\documentclass[journal,12pt,onecolumn,draftclsnofoot,]{IEEEtran}

\usepackage[colorlinks,urlcolor=blue,linkcolor=blue,citecolor=blue]{hyperref}

\usepackage{array}
\usepackage[hang]{footmisc}
\usepackage[subrefformat=parens,labelformat=parens,caption=false,font=footnotesize]{subfig}

\usepackage{amsmath}
\usepackage[nolist,nohyperlinks]{acronym}
\usepackage{amsmath}
\DeclareMathOperator*{\argmax}{arg\,max}
\DeclareMathOperator*{\argmin}{arg\,min}
	\begin{acronym}
    \acro{4G}{fourth generation}
    \acro{5G}{fifth generation}
    \acro{AoA}{angle of arrival}
   	\acro{AoD}{angle of departure}
   	\acro{AP}{access point}
    \acro{BCRLB}{Bayesian CRLB}
    \acro{BS}{base stations}
    \acro{BPP}{binomial point process}
    \acro{BP}{broadcast probing}
	\acro{CDF}{cumulative density function}
    \acro{CF}{closed-form}
    \acro{CRLB}{Cramer-Rao lower bound}
    \acro{ECDF}{empirical cumulative distribution function}
    \acro{EI}{Exponential Integral}
    \acro{eMBB}{enhanced mobile broadband}
    \acro{FIM}{Fisher Information Matrix}
    \acro{GoF}{goodness-of-fit}
    \acro{GPS}{global positioning system}
    \acro{GE}{group exploration}
    \acro{GNSS}{global navigation satellite system}
    \acro{HetNets}{heterogeneous networks}
    \acro{IoT}{internet of things}
    \acro{IIoT}{industrial internet of things}
    \acro{LOS}{line of sight}
    \acro{MAB}{multi-armed bandit}
    \acro{MBS}{macro base station}
    \acro{MEC}{mobile-edge computing}
    \acro{mIoT}{massive internet of things}
    \acro{MIMO}{multiple input multiple output}
    \acro{mm-wave}{millimeter wave}
    \acro{mMTC}{massive machine-type communications}
    \acro{MS}{mobile station}
    \acro{MVUE}{minimum-variance unbiased estimator}
    \acro{NLOS}{non line-of-sight}
    \acro{OFDM}{orthogonal frequency division multiplexing}
    \acro{PDF}{probability density function}
    \acro{PGF}{probability generating functional}
    \acro{PLCP}{Poisson line Cox process}
    \acro{PLT}{Poisson line tessellation}
    \acro{PLP}{Poisson line process}
    \acro{PPP}{Poisson point process}
    \acro{PV}{Poisson-Voronoi}
    \acro{QoS}{quality of service}
    \acro{RAT}{radio access technique}
    \acro{RL}{reinforcement-learning}
    \acro{RSSI}{received signal-strength indicator}
    \acro{BS}{base station}
   	\acro{SINR}{signal to interference plus noise ratio}
    \acro{SNR}{signal to noise ratio}
    \acro{SWIPT}{simultaneous wireless information and power transfer}
    \acro{TS}{Thompson Sampling}
    \acro{TS-CD}{TS with change-detection}
    \acro{KS}{Kolmogorov-Smirnov}
    \acro{UCB}{upper confidence bound}
	\acro{ULA}{uniform linear array}
	\acro{UE}{user equipment}
 	\acro{URLLC}{ultra-reliable low-latency communications}
    \acro{V2V}{vehicle-to-vehicle}    
    \acro{wpt}{wireless power transfer}
\end{acronym}

\usepackage{graphicx,wrapfig,lipsum}
\usepackage{rotating}
\usepackage{amsmath, amssymb, amsthm}
\usepackage{stmaryrd}
\usepackage{tikz}
\usepackage{verbatim}
\usepackage{url}
\usepackage[utf8]{inputenc}
\usepackage[english]{babel}
\newtheorem{theorem}{Theorem}[]

\newtheorem{condition}{Condition}[]
\newtheorem{lemma}[]{Lemma}

\usepackage{multicol}

\newtheorem{assumption}{Assumption}[]
\usepackage{algorithm}
\usepackage{algpseudocode}
\usepackage{scalerel}
\usepackage{multicol}
\usepackage{setspace}
\newtheorem{definition}{Definition}
\usepackage[noadjust]{cite}
\usepackage{booktabs}
\usepackage{bbm}
\usepackage[normalem]{ulem}
\usepackage{color}
\usepackage{dsfont}
\usepackage{bm}
\usepackage{setspace}
\usetikzlibrary{automata}
\usetikzlibrary{arrows}
\usetikzlibrary{shapes.geometric, arrows}
\usepackage[]{footmisc}
\usetikzlibrary{positioning}
\usetikzlibrary{calc}
\usetikzlibrary{arrows}
\hypersetup{nolinks=true}
\allowdisplaybreaks

\usepackage{caption}
\usepackage{mathtools}
\begin{document}

	\title{\fontsize{22.8}{27.6}\selectfont Fast Change Identification in Multi-Play Bandits and its Applications in Wireless Networks}
	\author{Gourab Ghatak, {\it Member, IEEE} 
	\thanks{The author is with the Department of Electrical Engineering at IIT Delhi, India 110016. Email: gghatak@ee.iitd.ac.in.}
    }
\date{}

\maketitle
\vspace{-2cm}		

\begin{abstract}
Next-generation wireless services are characterized by a diverse set of requirements, to sustain which, the wireless access points need to probe the users in the network periodically. In this regard, we study a novel multi-armed bandit (MAB) setting that mandates probing all the arms periodically while keeping track of the best current arm in a non-stationary environment. In particular, we develop \texttt{TS-GE} that balances the regret guarantees of classical Thompson sampling (TS) with the broadcast probing (BP) of all the arms simultaneously in order to actively detect a change in the reward distributions. The main innovation in the algorithm is in identifying the changed arm by an optional subroutine called group exploration (GE) that scales as $\log_2(K)$ for a $K-$armed bandit setting. We characterize the probability of missed detection and the probability of false-alarm in terms of the environment parameters. We highlight the conditions in which the regret guarantee of \texttt{TS-GE} outperforms that of the state-of-the-art algorithms, in particular, \texttt{ADSWITCH} and \texttt{M-UCB}. We demonstrate the efficacy of \texttt{TS-GE} by employing it in two wireless system application - task offloading in mobile-edge computing (MEC) and an industrial internet-of-things (IIoT) network designed for simultaneous wireless information and power transfer (SWIPT).
\end{abstract}

\section{Introduction}
Sequential decision making problems in \ac{RL} are popularly formulated using the \ac{MAB} framework, wherein, an agent (or player) selects one or multiple options (called arms) out of a set of arms at each time slot~\cite{9469216, 9113431, slivkins2019introduction, thompson1933likelihood, gittins2011multi}. Each time the player selects an arm or a group of arms, it receives a reward characterized by the reward distribution of the played arm/arms. The player performs consecutive action-selections based on its current estimate of the reward (or that of its distribution) of the arms. The player updates its belief of the played arms based on the reward received. In case the reward distribution of the arms is stationary, several algorithms have been shown to perform optimally~\cite{contal2013parallel}. On the contrary, most real-world applications such as \ac{IoT} networks~\cite{uprety2020reinforcement}, wireless communications~\cite{9767545}, computational advertisement~\cite{huh2020advancing}, and portfolio optimization~\cite{ghatak2021kolmogorovsmirnov} are better characterized by non-stationary rewards. However, non-stationarity in reward distributions are notoriously difficult to handle analytically.

To address non-stationarity, researchers either i) construct passively-adaptive algorithms that are change-point agnostic and work by discounting the impact of past rewards gradually or ii) derive frameworks to actively detect the changes in the environment. Among the actively-adaptive algorithms, the state-of-the-art solutions, e.g., \texttt{ADSWITCH} by Auer {\it et al.}~\cite{auer2019adaptively} provide a regret guarantee of $\mathcal{O}\left(\sqrt{KN_{\rm C}T\log T}\right)$ for a $K$-armed bandit setting experiencing $N_{\rm C}$ changes in a time-horizon $T$. Recently, researchers have also explored predictive sampling frameworks for tackling non-stationarity~\cite{liu2022nonstationary}. There the authors developed an algorithm that deprioritizes investment in acquiring information that will quickly
lose relevance.

Contrary to most non-stationary bandit algorithms, in several applications, the agent can possibly select multiple arms simultaneously. In such cases, the agent may either have access to the individual rewards of the arms played or a function of the rewards. For example, authors in \cite{ahmad2009multi} explored a wireless communication setup where a user can select multiple channels to sense and access simultaneously. Similarly, in portfolio optimization~\cite{liu2012learning} multiple plays refers to investing in multiple financial instruments simultaneously. However, to the best of our knowledge, none of the existing works investigate multiple simultaneous plays under a changing environment.

In contrast to this we propose an algorithm based on grouped probing of the arms that identifies the arm that has undergone a change in its mean. We investigate the conditions under which the proposed algorithm achieves superior regret guarantees than the state-of-the-art algorithms.

\subsection{Related Work}

{\bf Non-stationarity} Actively-adaptive algorithms have been experimentally shown to perform better than the passively-adaptive ones~\cite{mellor2013thompson}. In particular, \texttt{ADAPT-EVE}, detects abrupt changes via the Page-Hinkley statistical test (PHT)~\cite{hartland2006multi}. However, their evaluation is empirical without any regret guarantees. Similarly, another work~\cite{ghatak2021kolmogorovsmirnov} employs the Kolmogorov-Smirnov statistical test to detect a change in distribution of the arms. Interestingly, tests such as the PHT has been applied in different contexts in bandit frameworks, e.g., to adapt the window length of \texttt{SW-UCL}~\cite{srivastava2014surveillance}. The results by \cite{yu2009piecewise, ghatak2020change}, and by Cao {\it et al.}~\cite{cao2019nearly} detect a change in the empirical means of the rewards of the arms by assuming a constant number of changes within an interval. While the algorithm in~\cite{cao2019nearly}, called \texttt{M-UCB} achieves a regret bound of $\mathcal{O}(\sqrt{KN_{\rm C}T\log T})$, the work by Yu~{\it et al.} leverages side-information to achieve a regret of $\mathcal{O}(K\log T)$. However, the proposed algorithms in both these works assume a prior knowledge of either the number of changes or the change frequency. On these lines, recently, Auer {\it et al.}~\cite{auer2019adaptively} have proposed \texttt{ADSWITCH} based on the mean-estimation based checks for all the arms. Remarkably, the authors show regret guarantees of the order of $\mathcal{O}(\sqrt{KN_{\rm C}T\log T})$ for \texttt{ADSWITCH} without any pre-condition on the number of changes $N_{\rm C}$ for the $K$-arm bandit problem. If the number of changes $N_{\rm C}$ is known, a safeguard against a change in an inferior arm $a_i$ can be achieved by sampling it in inverse proportion of its sub-optimality gap $\Delta_i$. This achieves a regret of $\mathcal{O}(\sqrt{KLT})$. However, if the number of changes is not known, setting the sampling rate is challenging and if the number of changes is greater than $\sqrt{T}$, several algorithms experience linear regret. In order to avoid this, the main idea in \texttt{ADSWITCH} is to draw consecutive samples from arms, thereby incurring a regret that scales as $\mathcal{O}\left(\sqrt{K}\right)$ in the worst-case. Nevertheless, since both \texttt{M-UCB} and \texttt{ADSWITCH} provide the same regret guarantees, we choose them as the competitor algorithms for our proposal.

{\bf Multiple Plays:} The paper by Anantharam {\it et al.}~\cite{anantharam1987asymptotically} is one of the first works that introduced the concept of multiple plays. Apart from the work by \cite{ahmad2009multi}, others that have investigated bandits with multiple plays include \cite{xia2016budgeted} which considered budgeted settings for multiple plays with random rewards and cost. In the context of distributed channel access, \cite{youssef2021resource} explored the setting where multiple agents accessing the same arm simultaneously. Its important to note that a multi-player setting is different than a single-player multi-play setting as considered in this paper. Multi-player bandits may lead to collisions which can be resolved in several ways, e.g., equal division of reward. We refer the reader to the work by Besson {\it et al.}~\cite{besson2018multi} for a discussion on this.
For multi-play bandits that include multiple players, the recent work by Wang {\it et al.}~\cite{wang2022multiple} consider a per-load reward distribution and prove a sample complexity lower bound for Gaussian distribution. Nevertheless, none of these existing works treat multiple plays under non-stationarity which can dramatically alter the policy of the player. In particular, a missed-detection due to the change occurring in a less sampled arm can lead to a worst-case linear regret.

{\bf Frequency of Probing:}
Zhou {\it et al.}~\cite{zhou2019joint} have studied a framework for joint status sampling and updating in \ac{IoT} networks. The authors formulated the problem as an infinite horizon average cost constrained Markov decision process. The policy for a single \ac{IoT} device is derived and a trade-off is revealed between the average age of information and the sampling and updating costs.
The work by Stahlbuhk {\it et al.}~\cite{stahlbuhk2021learning} considers a system consisting of a single transceiver pair and a set of communication channels to chose from. They devise a policy to minimize the queue length regret, for which they provide a regret guarantee of the order of $\mathcal{O}(\log T)$.
A rigorous analysis for the regret of age-of-information bandits was presented in~\cite{9559999}, where first it is shown that UCB and Thompson sampling are order-optimal for AoI bandits. Then a novel algorithm is proposed that performs the classical algorithms by making AoI aware decisions. However to the best of our knowledge, none of these works consider the environment to be non-stationary which limits their applicability in modern wireless systems.

\subsection{Motivation and Contribution}
Unlike \texttt{M-UCB} but similar to \texttt{ADSWITCH}, we consider a framework where the number of changes $N_{\rm C}$ is not known a-priori but to be fewer than $\sqrt{T}$. Furthermore, we target an additional requirement - the agent algorithm should guarantee that the age between two consecutive plays of each arm is bounded. Although the issue of age has previously been addressed in some works (e.g., see~\cite{9559999}), the solutions cater to stationary environments. On the contrary, in this work, under an assumption of the hard-core distance between two consecutive changes, we propose \texttt{TS-GE} which outperforms \texttt{ADSWITCH} and \texttt{M-UCB} under several regimes of $K$ and the time-horizon, $T$, while simultaneously satisfying the mandatory probing requirement. The major innovation in this paper is two-fold - i) by allowing simultaneous probing of multiple arms in a coded manner, we reduce the scaling of changed-arm identification form $\mathcal{O}\left(K\right)$ to $\mathcal{O}\left(\log_2(K)\right)$, and ii) by design, \texttt{TS-GE} guarantees that the last sample of each arm is not older than $\sqrt{T}$. Overall, the main contributions of this paper are:
\begin{itemize}
    \item We develop and characterize \texttt{TS-GE}, tuned for non-stationary environments with unknown number of change-points. The additional design guarantee of \texttt{TS-GE} is the periodic mandatory probing of all the arms. Although this is relevant for several applications, to the best of our knowledge, this requirement has not been treated previously in literature. By balancing the regret guarantees of stationary Thompson sampling with grouped probing of all the arms, \texttt{TS-GE} ensures an upper bound of $\sqrt{T}$ in the sampling age of each arm.
    \item We propose a coded grouping of the arms based on the arm indices and consequently, derive the probability of missed detection of change and the probability of false alarm and highlight the conditions to limit these probabilities. Based on this, we show that \texttt{TS-GE} achieves sub-linear regret, $\mathcal{O}\left( K\log T + \sqrt{T}\left[\max\{N_{\rm C}\left(1 + \log K\right), T^{\frac{2}{5}}\}\right]\right)$. We compare this bound with the best known bound of $\mathcal{O}(\sqrt{KN_{\rm C}T\log T})$ and discuss the conditions under which the bound of \texttt{TS-GE} outperforms the latter. A summary of the pros and cons of our proposal is presented in Table~\ref{tab:comp}.
    \item As the first case-study, we employ \texttt{TS-GE} as a strategy for wireless nodes to offload their computational tasks to a set of \ac{MEC} servers. We demonstrate how the grouped probing phase of the algorithm enables the users to keep track of the best server to offload their tasks under a dynamic environment.
    \item Finally, as the second case-study, we consider an \ac{IIoT} network where a central controller is required to sustain \ac{SWIPT} services to the IoT devices. The different phases of \texttt{TS-GE} are mapped to the data-transfer and energy-transfer operations of the network. We demonstrate the performance of \texttt{TS-GE} with respect to the statistical upper-bound derived using stochastic geometry tools. Contrary to our proposal, in other algorithms such as \texttt{M-UCB} the exists non-zero probability that the IoT devices do not receive any energy transfer and hence our proposal will find applications that constrain sample age. 
\end{itemize}
The rest of the paper is organized as follows. In Section~\ref{sec:PS} we describe the system model and the problem setup. The proposed algorithm \texttt{TS-GE} algorithm is presented in Section~\ref{sec:ADF}. The mathematical analysis of missed detection, false alarm, and the regret analysis is presented in Section~\ref{sec:A}. Then, Section~\ref{sec:Reg} discusses the salient features of the derived regret bound. In order to test the efficacy of the proposed framework, we employ \texttt{TS-GE} in two case studies in Section~\ref{sec:CS1} and Section~\ref{sec:CS2}. Finally, the paper concludes in Section~\ref{sec:Con}. The notations used in the paper are summarized in Table~\ref{tab:notations}
\begin{table*}[t]
\small{
\centering
\begin{tabular}{ |c |c |c| }
\hline
\textbf{Feature} & \textbf{This paper} & \textbf{Benchmark~\cite{cao2019nearly, auer2019adaptively}}\\ \hline
Scaling with $K$        & \textcolor{blue}{$\mathcal{O}(\log_2 K)$}  (after initialization)   & \textcolor{red}{$\sqrt{K}$}\\
Guarantee on change-detection latency        & \textcolor{blue}{$\mathcal{O}(\sqrt{K})$}    & \textcolor{red}{No guarantee}\\
Scaling with $T$    & \textcolor{red}{Worst case - $\mathcal{O}(T^{0.9})$}    & \textcolor{blue}{$\mathcal{O}(\sqrt{T\log T})$}\\
Scaling with $N_{\rm C}$    & \textcolor{red}{$\mathcal{O}(N_{\rm C})$}    & \textcolor{blue}{$\mathcal{O}(\sqrt{N_{\rm C}})$}\\
Additional requirement    & Multiple-simultaneous plays    & -\\
\hline
\end{tabular}
}
\caption{Comparison of regret guarantees}
\label{tab:comp}
\end{table*}

\begin{table*}[t]
\centering
\begin{tabular}{ |c |c | }
\hline
\textbf{Symbol} &\textbf{Definition} \\
\hline 
$\mathcal{K}$, $K$ &Set of arms, number of arms. \\
$a_i$, $\mu_i(t)$, $\hat{\mu}_i(t)$ &$i-$th arm, its mean at time $t$, estimate of the mean at $t$\\
$T$ &Time horizon \\
$N_{\rm C}$ &Number of changes \\
$\pi(t)$ &Arm pulled at time $t$\\
$R_{i}(t)$ &Reward of arm $a_i$ at time $t$\\
$R_{\pi}(t)$ &Reward obtained by policy $\pi$ at time $t$\\
$R_{\pi, {\rm B}}(t)$ &Success/Failure event observed by policy $\pi$ at time $t$\\
$p_{\rm C}$ &Probability that an episode experiences a change\\
$p_{\rm b}$ &Probability that one slot experiences a change\\
$\mathcal{R}(T)$ &Regret experienced by a policy at time $T$\\
$T_{\rm l}$ &Length of each episode\\
$T_{\rm TS}$ &Length of each \ac{TS} phase\\
$T_{\rm BP}$ &Length of each \ac{BP} phase\\
$T_{\rm GE}$ &Length of each \ac{GE} phase\\
$\alpha_i, \beta_i$ &TS parameters of arm $a_i$\\
$B_k$ &$k-$th super arm\\
\hline
\end{tabular}
\caption{List of notations}
\label{tab:notations}
\end{table*}

\section{Problem Setup}
\label{sec:PS}
\begin{figure*}
    \centering
    \includegraphics[width = 0.8\textwidth]{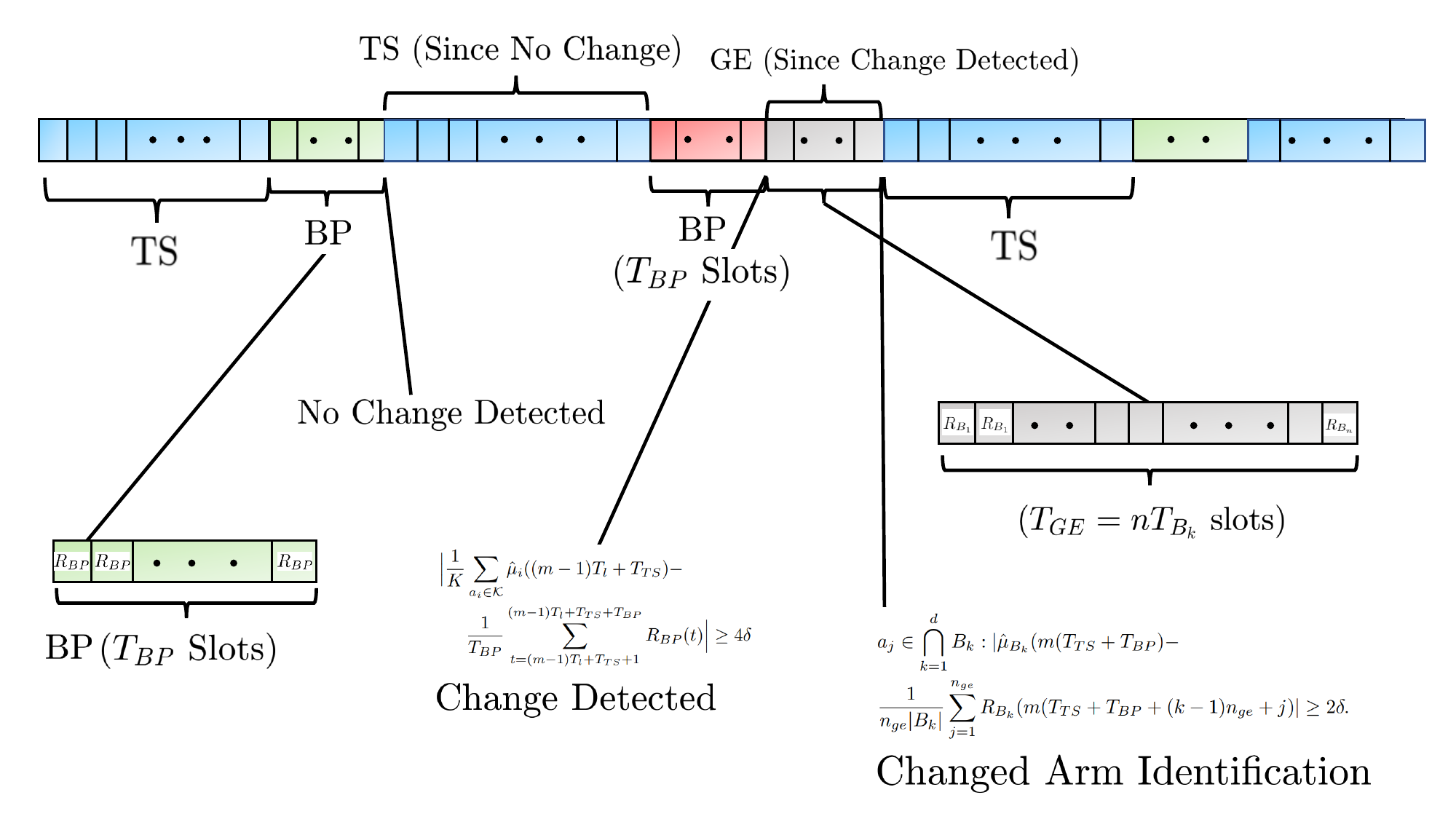}
    \caption{Illustration of the different phases of \texttt{TS-GE} after the initialization with \texttt{ETC} for \texttt{TS-GE}.}
    \label{fig:Illustration}
\end{figure*}
Consider a $K$-arm bandit setting with arms $a_i \in \mathcal{K}$, where $i = 0, 2, \ldots, K-1$. For the discussion in this paper, let us assume that $K = 2^d$ for some $d \in \mathbb{Z}^+$. It may be noted that in case the number of arms is not a power of 2, the same can be transformed into one by adding {\it dummy} arms which are sub-optimal with a probability 1 (e.g., arms with a constant reward of $-\infty$). The policy of the player is denoted by $\pi$. In other words ${\pi}(t) = a_i$ refers to the player pulling arm $a_i$ at time $t$. 

{\bf Reward distribution:} The reward $R_i(t)$ of an arm $a_i$ at time $t$ is assumed to be an instance of a truncated Gaussian random variable $\mathcal{N}\left(\mu_i(t), \sigma^2, R_{\rm max}\right)$ with mean $\mu_i(t)$, variance $\sigma^2$, and truncation limits $0$ and $R_{\rm max}$~\cite{burkardt2014truncated}. This assumption is valid for most practical wireless applications, e.g., upper bound on received power, \ac{SINR}, or data-rate by the wireless nodes. We note that the regret guarantees derived in this paper holds for any sub-Gaussian distribition of the rewards. The assumption of the trunctated Gaussian is 

It is worth to note that according to the above, the variance of all the arms is constant and is the same for all arms\footnote{This assumption is simply due to the ease of notations and the derivation of the regret bounds. The algorithm can be executed with this assumption does not hold. However, the regret analysis for such a case is much more involved and will be investigated in a future work.}, however, the mean is a function of time.

{\bf Success/Failure events:} Furthermore, we assume that in addition to the reward, the player observes a success or failure event $R_{\pi, {\rm B}}(t)$ corresponding to each $R_\pi(t) = R_i(t)$ given that $\pi(t) = a_i$. The value of $R_{\pi, {\rm B}}(t)$ is drawn from a Bernoulli distribution with parameter $\frac{R_\pi(t)}{R_{\rm max}}$. For example, in latency-constrained networks, $R_\pi(t)$ may correspond to the transmission \ac{SINR}, while $R_{\pi, {\rm B}}(t)$ corresponds to the event of successful packet reception which depends not only on $R_{\pi}(t)$ but also on other network parameters such as cell-load. It is precisely this uncertainty that we model using the Bernoulli trial. We note that the agent has access to both $R_\pi(t)$ and $R_{\pi, {\rm B}}(t)$. The availability of the $R_{\pi, {\rm B}}(t)$ with the player naturally leads to the choice of Beta priors in the \ac{TS} phase of our algorithm.

{\bf Multiple plays:} Finally, we assume that at any time-slot, multiple arms can be played by the agent. However, in that case, the agent observes a weighted average of the rewards of the pulled arms. This is formally mentioned below.
\begin{assumption}
In case at any time-slot $t$ the policy $\pi$ pulls multiple (say $n$) arms, i.e., $\pi(t) = \bigcup\limits_{k \in \mathcal{S}} a_k$, then the reward observed by the player is
$
    R_\pi(t) = \frac{1}{n}\sum_{a_k \in \mathcal{S}} R_{a_k}(t).$
\label{assump:multiplay}
\end{assumption}
Note that the player does not have access to the individual arm rewards of the set of arms it has played. This model fits into several application setting, for example, when the bandwidth $B$ is divided among $n$ users equally, the network throughput is $B/n$ times the sum of spectral efficiency experienced by each user. At this stage, we note that such an assumption is not present in~\cite{auer2019adaptively} as highlighted in Table~\ref{tab:comp}.

\subsection{Non-stationarity model}
{\bf Mandatory probing requirement:} The player interacts with the bandit framework in a sequence of $N_{\rm l}$ {\it episodes}, denoted by $E_i$, $i =  1, 2, \ldots, N_{\rm l}$, each of length $T_{\rm l}$. Consequently, the total time-horizon $T$ can be expressed as $T = N_{\rm l} T_{\rm l}$. It is important to note that the episodes are a feature of the problem setting and the framework rather than that of the algorithm. The episodes impose a condition on any feasible policy of the player as stated below.
\begin{condition}
The framework mandates that each arm be probed (either individually or in a group) at least once in each episode.
\end{condition}
This condition corresponds to applications such as mandatory status-updates or timely wireless power transfer in IoT networks. 

{\bf Change model:} We assume a piece-wise stationary environment in which changes in the reward distribution occur at time slots called {\it change points}, denoted by $t_{C_j}$, $j = 1, 2, \ldots$, where each $T_{C_j} \in [T]$. At each change point, exactly one of the arms $a_i$, uniformly selected from $\mathcal{K}$ experiences a change (increase or decrease) in its mean by an unknown amount $\Delta_{{\rm C},i}$.
Furthermore, we assume that during each episode, at most one change point occurs with a probability $p_{\rm C}$ and the total number of changes is $N_{\rm C}$ within $T$, which is unknown to the player. In particular, at in each episode, at the beginning of each time slot, the environment samples a Bernoulli random variable $C$ with success probability $p_{\rm b}$. In case of a success, the change occurs in that slot, while in case of a failure, the bandit framework does not change. Once a change occurs in an episode, the change framework is paused until the next episode. Thus,
\begin{align}
    &\mathbb{P}\left(\text{Episode } E_i \text{ experiences $c_i$ changes}\right)
  = \begin{cases}
p_{\rm C}; &c_i = 1, \nonumber \\
1 - p_{\rm C}; &c_i = 0, \nonumber \\
0; &c_i > 1.
\end{cases}
\end{align}
where, $p_{\rm C} = \sum_{k = 1}^{T_{\rm l}} (1 - p_{\rm b})^{k-1} p_{\rm b}$. As a practical analogy, this change framework can be related to the state of a \ac{MEC} server admitting new tasks at each episode. Let an edge device attempt to offload its task to the \ac{MEC} server at each time slot. The event of successful admission of the task depends on the server state, the channel conditions, the computational requirement, etc. On an offload failure, let the edge device attempt an offload again in the next slot wherein the server state and the computational requirement is assumed to remain same, while the channel state can be considered to be and identical and independent realization of the previous slot. The probability of a successful offload can be assumed to be $p_{\rm b}$. Once such a new task is admitted in the server, its state changes and it freezes further admitting new tasks until the rest of the episode. This will be discussed in detail in Section~\ref{sec:CS2}.



\subsection{Policy Design}
The challenge for the player is to {\it quickly} identify changes that any arm has undergone and adapt its corresponding parameters. For benchmarking the performance of a candidate player policy, at a given time slot, a policy $\pi$ competes against a policy class which selects the arm with the maximum expected reward at that time slot. Thus, any policy $\pi$ that intends to balance between the exploration-exploitation trade-off of the bandit framework experiences a regret 
\begin{align}
    \mathcal{R}(T) = \sum_{t = 1}^T\max_i\mu_{i}(t) -\mathbb{E}\left[\mu_\pi(t) \right],
\end{align}
where $\mu_\pi(t)$ is the mean of the arm $a_\pi(t)$ picked by the policy $\pi$ at time $t$. It can be noted here that unlike stationary environments, the identity of the best arm $a_j$, where, $j = \argmax_i \mu_i(t)$ is not fixed and may change with at each change point.

\section{\texttt{TS-GE}: Algorithm Description and Features}
\label{sec:ADF}

The key features of our proposed algorithm \texttt{TS-GE} are i) actively detecting the change in the bandit framework, ii) identifying the arm which has undergone a change, and iii) modify its probability of getting selected in the further rounds based on the amount of change. The \texttt{TS-GE} algorithm consists of an initialization phase called explore-then-commit, \texttt{ETC}, followed by two alternating phases: classical \ac{TS} phase followed by a \ac{BP} phase to determine a change in the system. In case a change is detected in the \ac{BP} phase, the arm which has undergone a change is identified using an optional sub-routine called \ac{GE}. Thus, the \ac{GE} phase is only triggered if a change is detected in the BP phase. The overall algorithm is illustrated in Fig.~\ref{fig:Illustration} and presented in Algorithm~\ref{alg:main}. Each $E_i$ consists of one \texttt{TS} phase, one \texttt{BP} phase, and (optionally) one \texttt{GE} phase.

\subsection{Initialization: \texttt{ETC} for \texttt{TS-GE}}
For initialization, the player performs a \texttt{ETC} for \texttt{TS-GE}, wherein each arm is played an $n_{\rm ETC}$ number of times and consequently, their mean $\mu_i$ is estimated to be $\hat{\mu}_i$.
\begin{definition}
An arm $a_i$ is defined to be well-localized if the empirical estimate $\hat{\mu}_i$ of its mean $\mu_i$ is bounded as $
    |\hat{\mu}_i(t) - \mu_i(t)| \leq \delta.$
\end{definition}
Let $p_{\rm L}$ be the probability that an arm $i$ is well-localized. We will characterize $n_{\rm ETC}$ and $p_{\rm L}$ at a later stage. After the \texttt{ETC} for \texttt{TS-GE}, the player switches to alternating \ac{TS} and \ac{BP} phases as discussed below.

\subsection{Alternating \ac{TS} and \ac{BP} phases}
Each episode consists of a \ac{TS} phase, a \ac{BP} phase, and an optional \ac{GE} phase. Since in practical scenarios of interest, the time-horizon $T$ can be made arbitrary, let us set $T_{\rm l} = \sqrt{T}$. In the \ac{TS} phase, the player performs the action selection of the choices according to the \ac{TS} algorithm for $T_{\rm TS}$ slots as given in~\cite{agrawal2012analysis}. The version of Agrawal and Goyal~\cite{agrawal2012analysis} can be implemented here due to the fact that the player has access to the binary rewards $R_{\pi, {\rm B}}(t)$.  Each arm $a_i$ is characterized by its \ac{TS} parameters $\alpha_i$ and $\beta_i$, all of which are initially set to one\footnote{Note that the TS parameters can be initialized according to the observed rewards in the \texttt{ETC} phase. However, in order to safeguard against changes in the \texttt{ETC} phase, we begin the TS phase with fresh parameters.}. After each play of an arm $a_i$, its estimated mean $\hat{\mu}_i$ is updated according to the reward from the truncated Gaussian distribution $R_{\pi}(t)$. Let $a_i$ be played at time slots $\{t_i\} \in [T]$ and the number of times it is played is $n_i(t)$ until (and including) the time-slot $t$, then:
\begin{align}
    \hat{\mu}_i(t) = \frac{1}{n_i(t)}\sum_{t \in t_i} R_i(t)
\end{align}
Additionally, the TS parameters for the played arm $a_i$, i.e., $\alpha_i$ and $\beta_i$ are updated as per a Bernoulli trial with a success probability $R_{\pi, {\rm B}}(t), \forall t \in \{t_i\}$ each time $a_i$ is played.

Each \ac{TS} phase is followed by the {BP} phase for $T_{\rm BP}$ time-slots, where the player samples all the arms simultaneously for $T_{\rm BP}$ rounds. During this phase, the reward observed by the player is the average of the rewards from all the arms, i.e., $\pi(t) = \mathcal{K}$, for all the time slots in the BP phase. The reward in the BP phase is then compared with the average of the estimates of all the arms to detect whether an arm of the framework has changed its mean. Recall that as per Assumption~\ref{assump:multiplay}, during the BP phase of the $m$-th episode, the player receives the following reward for each play:
\begin{align}
    R_{BP}(t) 
    \sim \mathcal{N}_{\rm T}\left(\frac{1}{K}\sum_{a_i \in \mathcal{K}} \mu_i(t), \frac{\sigma^2}{K}\right), \nonumber \forall (m-1) T_{\rm l} + T_{\rm TS} < t \leq (m-1) T_{\rm l} + T_{\rm TS} + T_{\rm BP}  \nonumber 
\end{align}
where $\mathcal{N}_{\rm T}\left(x, y\right)$ represents the truncated Gaussian distribution with mean $x$ and variance $y$. At the end of the $m$-th BP phase, a change is detected if:
\begin{align}
    \Big|\frac{1}{K}\sum_{a_i \in \mathcal{K}} \hat{\mu}_i((m-1)T_{\rm l} + T_{\rm TS}) - 
    \frac{1}{T_{\rm BP}}\sum_{t = (m-1)T_{l} + T_{\rm TS} + 1}^{(m-1)T_{\rm l} + T_{\rm TS} + T_{\rm BP}} R_{BP}(t)\Big|  \geq 4\delta 
    \label{eq:BPchange}
\end{align}
Here the first term is the average of the estimated means of all the arms at the end of the $m$-th TS phase, while the second term represents the same evaluated during the $m$-th BP phase. In case the change does not occur or goes undetected, the algorithm continues with the next TS phase. However, in case a change is detected or a false-alarm is generated, the algorithm moves on to the \ac{GE} sub-routine as described below.

\subsection{Policy after change detection}
If a change is detected in the \ac{BP} phase, the \ac{GE} phase begins for the identification of the changed arm. The key step in this phase is the creation of $d$ sets $B_k \subset \mathcal{K}, k = 1, 2, \ldots, d$, called {\it super arms} as shown in Algorithm 2. Recall that $d$ is a number such that $d = \log_2(K)$. It may be noted that an optimal grouping of arms may be derived that considers the fact that the arms that have been played a fewer number of times have a larger error variance of its mean estimate. However, such a study is out of scope for the current text and will be treated in a future work.

{\bf Arm Grouping Strategy:} The $i$-th arm, $i = 0, 1, 2, \ldots, K-1$ is added to a super arm $B_k$, $k = 1, 2, \ldots, d$ if and only if the binary representation of $i$ has a "1" in the $k$-th binary place. In other words, $a_i$ is added to $B_k$ if:
\begin{align}
    \texttt{bin2dec}\left(\texttt{dec2bin}(i) \; \texttt{ AND } \; \texttt{onehot}(k) \right) \neq 0
    \nonumber 
\end{align}
where $\texttt{bin2dec}()$ and $\texttt{dec2bin}()$ are respectively operators that convert binary numbers to decimals and decimal numbers to binary. Additionally, $\texttt{onehot}(k)$ is a binary number with all zeros except 1 at the $k$-th binary position. $\texttt{AND}$ is the bit-wise AND operator. 
This strategy for creating super-arms is inspired from the classical forward error correcting strategy due to Hamming, which enables detection and correction of single-bit errors~\cite{hamming1950error}. 
In the \ac{GE} phase, each super arm is played $n_{\rm ge}$ times and the player obtains a reward which is the average of rewards of all the arms that belong to $B_k$ as per Assumption~\ref{assump:multiplay}, i.e., each time the super arm $B_k$ is played, the player gets a reward that is sampled from the distribution:
\begin{align}
    R_{B_k} \sim \mathcal{N}_{\rm T}\left(\frac{1}{|B_k|} \sum_{i \in B_k} \mu_i, \frac{\sigma^2}{|B_k|}\right) \nonumber 
\end{align}
Let the mean reward of the super arm $B_k$ be denoted by $\mu_{B_k}$. Before the beginning of each GE phase, $\mu_{B_k}$ is estimated. As an example, let a change be detected after the $m$-th BP phase. Then, the estimate of the mean reward of the super arm $B_k$ is:
\begin{align}
    \hat{\mu}_{B_k}(m(T_{\rm TS} + T_{\rm BP})) = \frac{1}{|B_k|} \sum_{a_i \in B_k} \hat{\mu}_i(m(T_{\rm TS} + T_{\rm BP}))
\end{align}

Then, the arm with the changed mean is the one that belongs to the all super arms in which a change of mean is detected. For each super arm $B_k$, the changed arm $a_j$ either is present in $B_k$ or it is present in its complimentary set ${B}_k^{\rm C}$. 
Let us define a new sequence of sets as
\begin{align}
    C_k = \begin{cases}
    B_k; \quad \text{If a change is detected in }B_k \\
    {B}_k^{\rm C}; \quad \text{If no change is detected in }B_k
    \end{cases}
\end{align}
Then, the changed arm is identified as $a_j$, where $a_j = \bigcup C_k$. In case no change is detected in any of the super arms, but the \ac{BP} phase detects a change, the changed arm is identified as $a_1$.

\begin{figure}
    \centering
    \includegraphics[width = 0.45\textwidth]{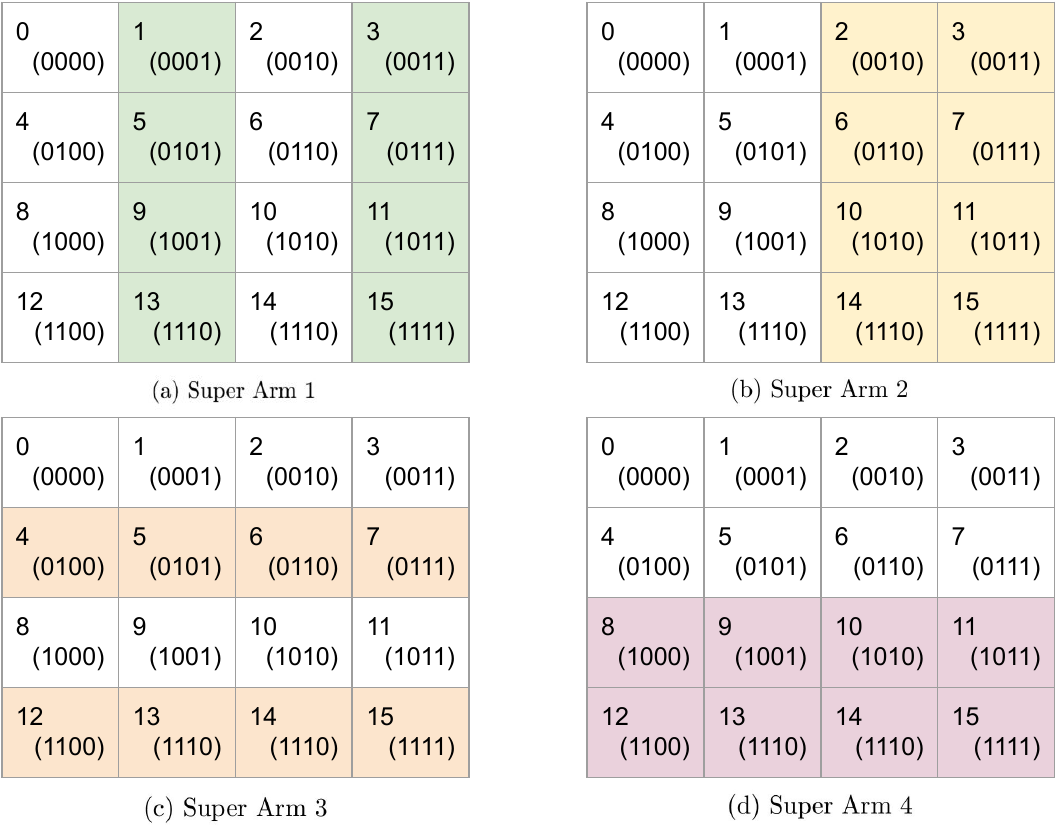}
    \caption{Illustrations of the for super arms for the case of $K = 16$. Here we have $d = 4$ and accordingly, 4 super arms. Each cell corresponds to an arm $a_i$ and shows the decimal and binary representation of $i$.}
    \label{fig:superarms}
\end{figure}
\subsection{Illustrative Example}
Let us elaborate this further with the example of an illustrative example with $K = 16$. Fig.~\ref{fig:superarms} shows the following grouping for the super arms -
\begin{itemize}
    \item $B_1$ - Arms with '1' in the first decimal place - $a_1, a_3, a_3, a_5, a_7, a_9, a_9, a_{11}, a_{13},$ and $a_{15}$.
    \item $B_2$ - Arms with '1' in the second decimal place - $a_2, a_3, a_6, a_7, a_{10}, a_{11}, a_{14},$ and $a_{15}$.
    \item $B_3$ - Arms with '1' in the third decimal place - $a_4, a_5, a_6, a_7, a_{12}, a_{13}, a_{14},$ and $a_{15}$.
    \item $B_4$ - Arms with '1' in the fourth decimal place - $a_8, a_9, a_{10}, a_{11}, a_{12}, a_{13}, a_{14},$ and $a_{15}$.
\end{itemize}
Now, in case a change occurs in an arm, say arm $a_{10}$, this will be detected in super arms $B_2$ and $B_4$ but not in $B_1$ and $B_3$. Additionally, this combination of changes in the super arms uniquely corresponds to the arm $a_{10}$. Similarly, a change only in the super arm $B_3$ and in no other super arms corresponds to a change in arm $a_3$. Finally, note that no changes in any of the super arms (but a change detected in the BP phase) corresponds to a change in $a_1$.

Once the change is detected and the changed arm $a_j$ is identified, the corresponding mean of $a_j$ is updated as
\begin{align}
    \hat{\mu}_j = \sum_{k: a_j \in B_k} \hat{\mu}_{B_k} - \sum_{k: a_j \in B_k} \sum_{a_i \in B_k: i\neq j} \hat{\mu}_i
\end{align}
Furthermore, the TS parameters of the arm is updated. In particular, we set the parameters of $a_j$ to be same as the arm that has an estimated mean closest ot $a_j$ as:
\begin{align}
    \alpha_j = \alpha_k, \quad \beta_j = \beta_k, \quad \text{where } k = \argmin_{i \neq j} |\hat{\mu}_i -  \hat{\mu}_j| \nonumber 
\end{align}

In the next section we characterize the probability with which \texttt{TS-GE} misses detection a change or raises a false alarm in case of no change. This eventually leads to the regret.
\begin{algorithm}
\caption{\texttt{TS-GE}}
\begin{algorithmic}[1]
\State {\bf Parameters:}
\State {\bf Initialization:} $\alpha_k = \beta_k = 1$, $\forall k = 1, \ldots K$. \\
\vspace{0.1cm}
{\bf Thompson Sampling Phase:}
\For{$e_i = E_1, \ldots, E_{N_{\rm l}}$}
\For{$t = 1, \ldots, T_N$}
    \State $\theta_i \sim \text{Beta}\left(\alpha_i, \beta_i\right)$. \qquad $\backslash \backslash$ Sample the Beta prior.
        \State $a_j \gets a_i | \theta_j = \max(\theta_i)$ \quad $\backslash \backslash$ Select the best arm.
        \State $R_{\texttt{TS-GE}}(t) \gets R_{a_j}(t)$ \quad $\backslash \backslash$ Reward at time $t$.
        \State $R_\pi(t) \gets   \frac{R_{a_j}(t)}{R_{\rm max}}$ \quad $\backslash \backslash$ Normalize for Beta update.
        \State $R^*$ = Bern $(R_\pi(t))$ \quad $\backslash \backslash$ Bernoulli trial for Beta update.
        \State $\alpha_j  \gets \alpha_j + 1 - R^*$ \quad $\backslash \backslash$ Update priors.
        \State $\beta_j  \gets \beta_j + R^*$ \quad $\backslash \backslash$ Update priors.
        \State $n_j \gets n_j + 1$ \quad $\backslash \backslash$ Count of arm $a_j$.
        \State $\hat{\mu}_j(t) = \frac{\sum_{s = 1}^{t} R_{a_j}(s) \mathcal{I}(a_j(s))}{n_j}$ \quad $\backslash \backslash$ Update estimated mean of $a_j$.
\EndFor
\State $p \gets p + 1$. \qquad $\backslash\backslash$End of the $p$-th TS phase. \\
\vspace{0.1cm}
{\bf Broadcast Probing Phase:}
\State Play all the arms simultaneously for $T_{\rm BP}$ rounds and build the estimate:
\begin{align}
  \hat{\mu}_{BP}  = \frac{1}{T_{\rm BP}}\sum_{t = (e_i-1)T_{l} + T_{\rm TS} + 1}^{(e_i-1)T_{\rm l} + T_{\rm TS} + T_{\rm BP}} R_{BP}(t) \nonumber 
\end{align}
\If{Equation \eqref{eq:BPchange} holds}
    \State Change is detected.\\
    {\bf Group Exploration Phase:}
\State Construct super-arms $\{B_k\} = \texttt{CSA}({\bf a})$.
\For{$T_{\rm GE}$ slots} 
    \State Play $B_k$ for $T_{B_k}$ rounds.
    \State $\backslash\backslash$ Update $\mu_{B_k}$:
    \begin{align}
       \hat{\mu}_{B_k}(e_i(T_{\rm TS} + T_{\rm BP})) = \frac{1}{|B_k|} \sum_{a_i \in B_k} \hat{\mu}_i(e_i(T_{\rm TS} + T_{\rm BP})) \nonumber 
    \end{align}
\EndFor
    \State  $a_j \in \bigcap\limits_{k = 1}^n B_k : |\hat{\mu}_{B_k} - \frac{1}{n} \sum_{i \in B_k} \hat{\mu}_i(pT_N)| \geq \Delta, k = 1, 2, \ldots, n.$, \quad $\backslash\backslash$ Identify changed arm.
    \State $\hat{\mu}_j(e_iT_{\rm l} + 1) = \sum_{k: a_j \in B_k} \hat{\mu}_{B_k} - \sum_{k: a_j \in B_k} \sum_{a_i \in B_k: i\neq j} \hat{\mu}_i$ \qquad $\backslash\backslash$ Update the changed arm.
    \State Update the Beta parameters of the changed arm:
    \begin{align}
    &\alpha_j = \alpha_k, \quad \beta_j = \beta_k \nonumber 
    \\ \text{where } &k = \argmin_i |\hat{\mu}_i -  \hat{\mu}_j| \nonumber 
\end{align}
    \Else
        \State Continue. \qquad $\backslash\backslash$ When no change is detected
\EndIf
\EndFor
\end{algorithmic}
\label{alg:main}
\end{algorithm}

\begin{algorithm}
\caption{Construct Super-Arms \texttt{CSA}}\label{alg:cap}
\begin{algorithmic}
\State {\bf Input:} ${\bf a}$ and ${\bf n}$.
\State {\bf Initialize:} $B_k = \{\}$, $\forall k = 1, 2, \ldots, K$.
\For{$k = 1$ to $n$}
    \For{$i = 0$ to $K - 1$}
        \If{\texttt{dec2bin}$(i)$ \texttt{AND} \texttt{onehot}$(k) \neq \texttt{zeros}(1,n)$}
            \State $B_k = B_k \cup a_i$
        \EndIf
    \EndFor
\EndFor
\State Return $B_k$
\end{algorithmic}
\end{algorithm}

\section{Analysis of \texttt{TS-GE}}
\label{sec:A}
In this section we derive some conditional mathematical results for \texttt{TS-GE}. In particular, under conditions on the magnitude of change and a lower bound on the probability of change in each slot, we derive a sub-linear regret for \texttt{TS-GE}. First, let us note the following for the \texttt{ETC} phase.
\begin{lemma}
In the stationary regime, in order for the arm $a_i$ to be well-localized with a probability $1 - p_{\rm L}$, the arm needs to have been played at least $n_{\rm ETC}$ times, where
\begin{align}
    n_{\rm ETC} = \frac{R_{\rm max}}{2\delta^2}\ln \frac{1}{p_{\rm L}} \nonumber
\end{align}
\end{lemma}
\begin{proof}
The proof follows from Hoeffding's inequality for bounded random variables between 0 and $R_{\rm max}$.
\end{proof}
Thus, the \texttt{ETC} phase lasts for at least $T_{\rm ETC} = K n_{\rm ETC} = \frac{KR_{\rm max}}{2\delta^2}\ln \frac{1}{p_{\rm L}}$ rounds. Naturally, in order to restrict $p_{\rm L}$ to $\mathcal{O}(\frac{1}{T})$, $n_{\rm ETC}$ needs to be $\mathcal{O}(\ln T)$. Then let us recall that the GE phase is triggered only if a change is detected in the BP phase. Consequently, the algorithm can miss detecting the chase in case the change occurs either in the TS phase or the BP phase, which we analyze below.

\subsection{Probability of Missed Detection}
Since the only condition on $T_{\rm TS}$ is that it has to be upper bounded by $T_{\rm l} = \sqrt{T}$, let us set $T_{\rm TS} = \sqrt{T} - T^{\frac{2}{5}}$. Let the change occur in the arm $a_i$ at $t_c$ time slots in the $m$-th TS phase, i.e., $T_{\rm ETC} + (m-1)(T_{\rm TS} + T_{\rm BP}) < t_c \leq T_{\rm ETC} + mT_{\rm TS} + (m-1)T_{\rm BP}$. The mean is assumed to change from $\mu_i^-$ to $\mu_i^+$. In other words, the distribution of the reward of $a_i$ is given as:
\begin{align}
    &R_i(t) \sim \nonumber 
    \begin{cases}
    X_i^- \sim  \mathcal{N}_{\rm T}\left(\mu_i^-, \sigma^2\right); \; \text{before change} \nonumber \\
    X_i^+ \sim \mathcal{N}_{\rm T}\left(\mu_i^+, \sigma^2\right); \; \text{after change}
    \end{cases}
\end{align}
The following lemma characterizes the probability of missed detection when the change occurs in the TS phase.

\begin{lemma}
Let the arm $a_i$ change its mean from $\mu_i^-$ to $\mu_i^+$, where $\Delta_{{\rm C},i} = \mu_i^+ - \mu_i^-$ at a time slot $t_c$ in the $m$-th TS phase. Then, if $\Delta_{{\rm C},i} \geq 2\sigma$, the probability of missed detection after the $m$-th BP phase following this change is upper bounded by
 $   \mathcal{P}_{\rm M}^{\rm TS} \leq \frac{1}{T}.$
\label{lem:PMTS}
\end{lemma}
\begin{proof}
Please see Appendix~\ref{app:MissTS}.
\end{proof}
Note that detecting relatively small changes is not only challenging but also becomes unnecessary in several practical applications. Accordingly, we assume a minimum quantum of change once such a change event occurs. The above result bounds the probability of missed detection in the \ac{TS} phase to inversely proportional to T. However, if the change occurs in the \ac{BP} phase, the probability of missed detection increases as discussed below.
\begin{lemma}
Let the arm $a_i$ change its mean from $\mu_i^-$ to $\mu_i^+$, where $\Delta_{{\rm C},i}$ at a time slot $t_c$ in the $m$-th BP phase. Furthermore, let $p_{\rm b}$ be lower-bounded as
\begin{align}
    p_{\rm b} \geq 1 - \left(\frac{1}{T}\right)^{\frac{1}{\sqrt{T} - T^{\frac{2}{5}}}}.\label{ass:pC}
\end{align}
Then the probability of missed detection after the $m$-th BP phase following this change has the following characteristic:
\begin{align}
    \mathcal{P}_{\rm M}^{\rm BP} = \begin{cases}
     \mathcal{P}_{M,Case 1}^{\rm BP}; & \text{with probability } > 1- \frac{1}{T} \\
     \mathcal{P}_{M,Case 2}^{\rm BP}; & \text{with probability} \leq \frac{1}{T},
    \end{cases}
\end{align}
where, $\mathcal{P}_{M,Case 1}^{\rm BP} \leq \frac{1}{T}$ and 
$\mathcal{P}_{M,Case 2}^{\rm BP} >  1 - \frac{1}{T}$.
\label{lem:PMBP}
\end{lemma}
\begin{proof}
Please see Appendix~\ref{app:MissBP}.
\end{proof}

First, let us note that since the right hand side of \eqref{ass:pC} is a decreasing function, for large values of $T$, it is a fairly mild assumption. This lemma shows that a change in the \ac{BP} phase results in a different probability of missed detection based on the exact point of change. In particular, from the proof we note that for cases 2 and 4, the  algorithm misses the detection of the change in the \ac{BP} phase with a high probability. However, these cases themselves occur with low probability due to the condition \eqref{ass:pC}, which enables us to bound the probability of missed detection. 

\subsection{Probability of False Alarm}
Another aspect of the algorithm that needs to be considered for a regret analysis is the fact that the BP phase can raise a false alarm when a change has not occurred in an episode while, the condition \eqref{eq:BPchange} holds true simultaneously. However, in case of no change, the test statistic is simply
    $Z_{\rm NC} \sim \mathcal{N}_{\rm T}\left(0,\sigma_{NC}\right)$,
where $\sigma^2_{NC} =  \frac{\sigma^2}{K}\left(\frac{1}{n_{\rm ETC}} + \frac{1}{mT_{\rm BP}} + \sum_{a_j \in \mathcal{K}} \frac{1}{n_j(m(T_{\rm TS} + T_{\rm BP}))} \right)$. Here $n_j(mT_{\rm l})$ is the number of times the arm $a_j$ has been played in all the TS phases. Thus,
\begin{align}
    \mathcal{P}_{FA} &= \mathbb{P}\left(|Z_{\rm NC}| \geq 4\delta\right) \leq \mathcal{Q}\left(\frac{4\delta}{{\sigma_{NC}}}\right) \leq \frac{1}{T}
    \label{eq:pfa}
\end{align}
\subsection{On the Regret of \texttt{TS-GE}}
Now we have all the necessary results to derive the following regret bound for \texttt{TS-GE}.

\begin{theorem}
The regret for \texttt{TS-GE} under assumption $\Delta_{\rm C} \geq2 \sigma$ and \ref{ass:pC} is upper bounded as
\begin{align}
   \mathcal{R}_\texttt{TS-GE} \leq\mathcal{O}\left( K\log T + \sqrt{T}\left[\max\{N_{\rm C}\left(1 + \log K\right), T^{\frac{2}{5}}\}\right]\right)
\end{align}
\label{theo:Regret}
\end{theorem}
\begin{IEEEproof}
Please see Appendix~\ref{app:Regret}
\end{IEEEproof}
Thus, not only the regret of \texttt{TS-GE} is sub-linear but also as discussed in the next section, it outperforms the known bounds under several regimes. Additionally, let us note that in most wireless system applications, the initialization phase occurs only once throughout a service session, e.g., initial access for wireless nodes or initial server allocation for edge devices. Hence, from a practical standpoint, the first term of the regret which scales linearly with $K$ does not impact long-term deployment of the algorithm.

\section{Discussion on Regret}
\label{sec:Reg}

\begin{figure*} 
    \centering
  \subfloat[\footnotesize $K = 100$.\label{fig:comparison_small}]{%
       \includegraphics[width=0.32\textwidth, height = 0.18\textheight]{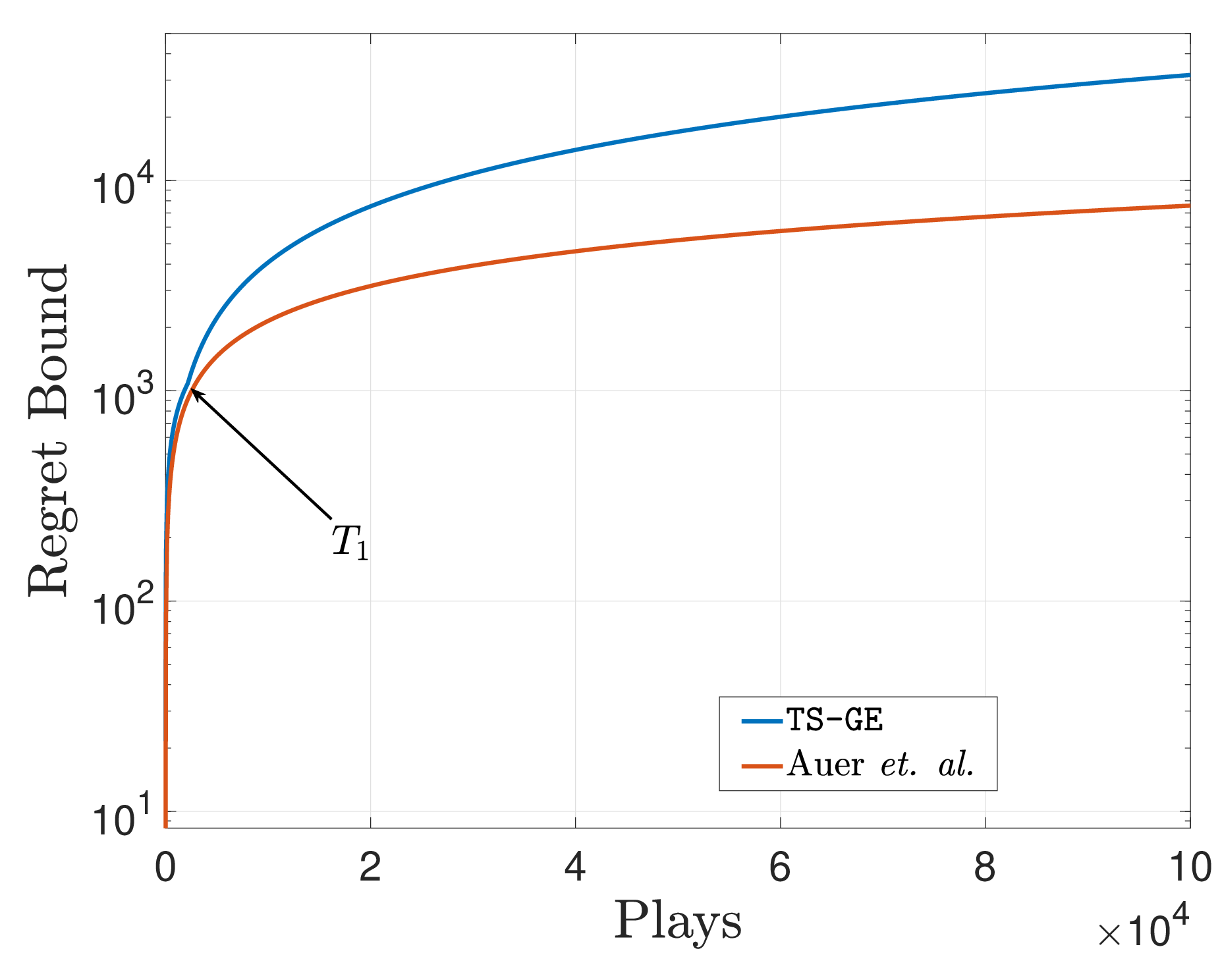}}
  \subfloat[\footnotesize$K = 500$. \label{fig:comparison_mid}]{%
        \includegraphics[width=0.32\textwidth, height = 0.18\textheight]{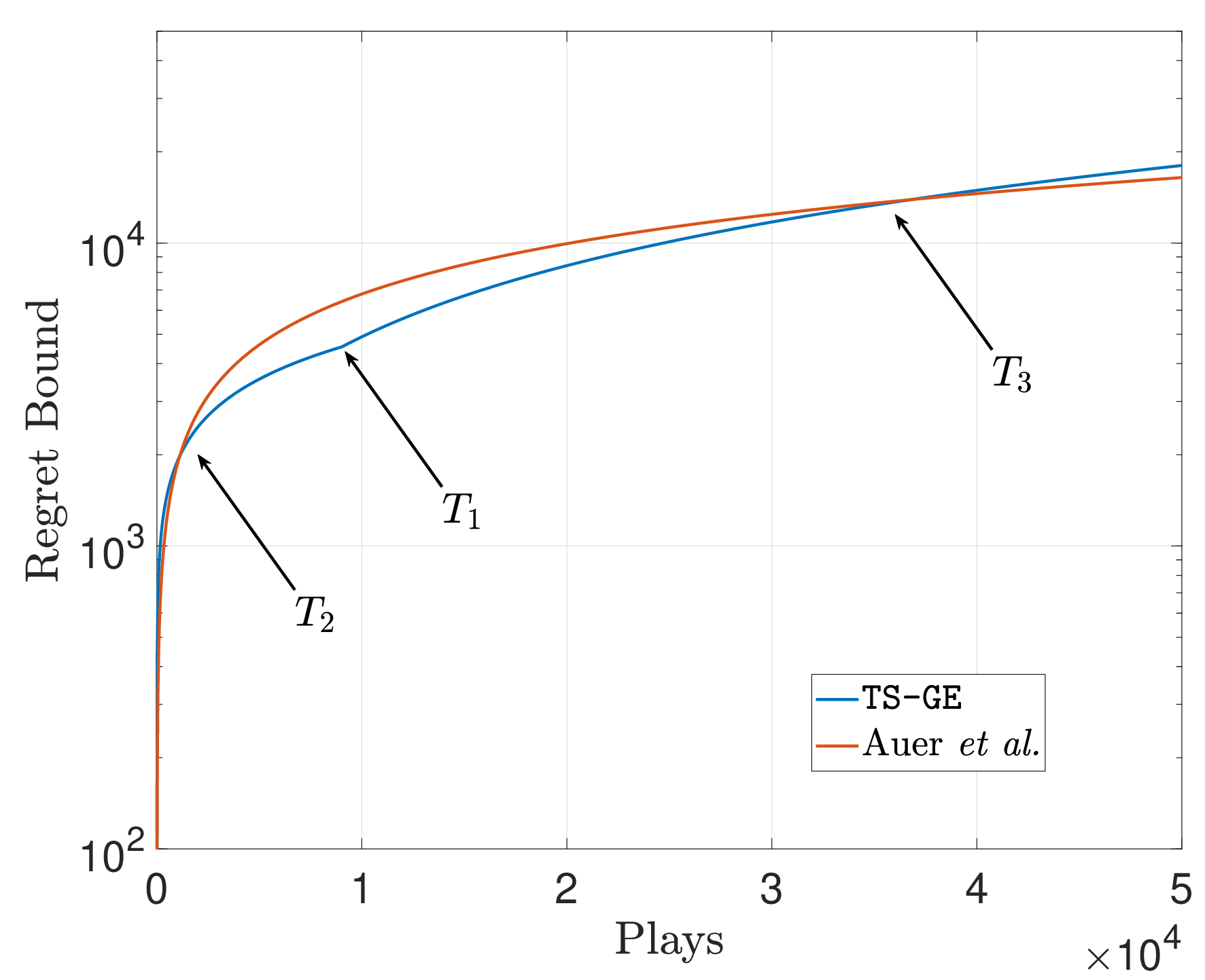}}
 \subfloat[\footnotesize $K = 1000$.\label{fig:comparison_large}]{%
        \includegraphics[width=0.32\textwidth, height = 0.18\textheight]{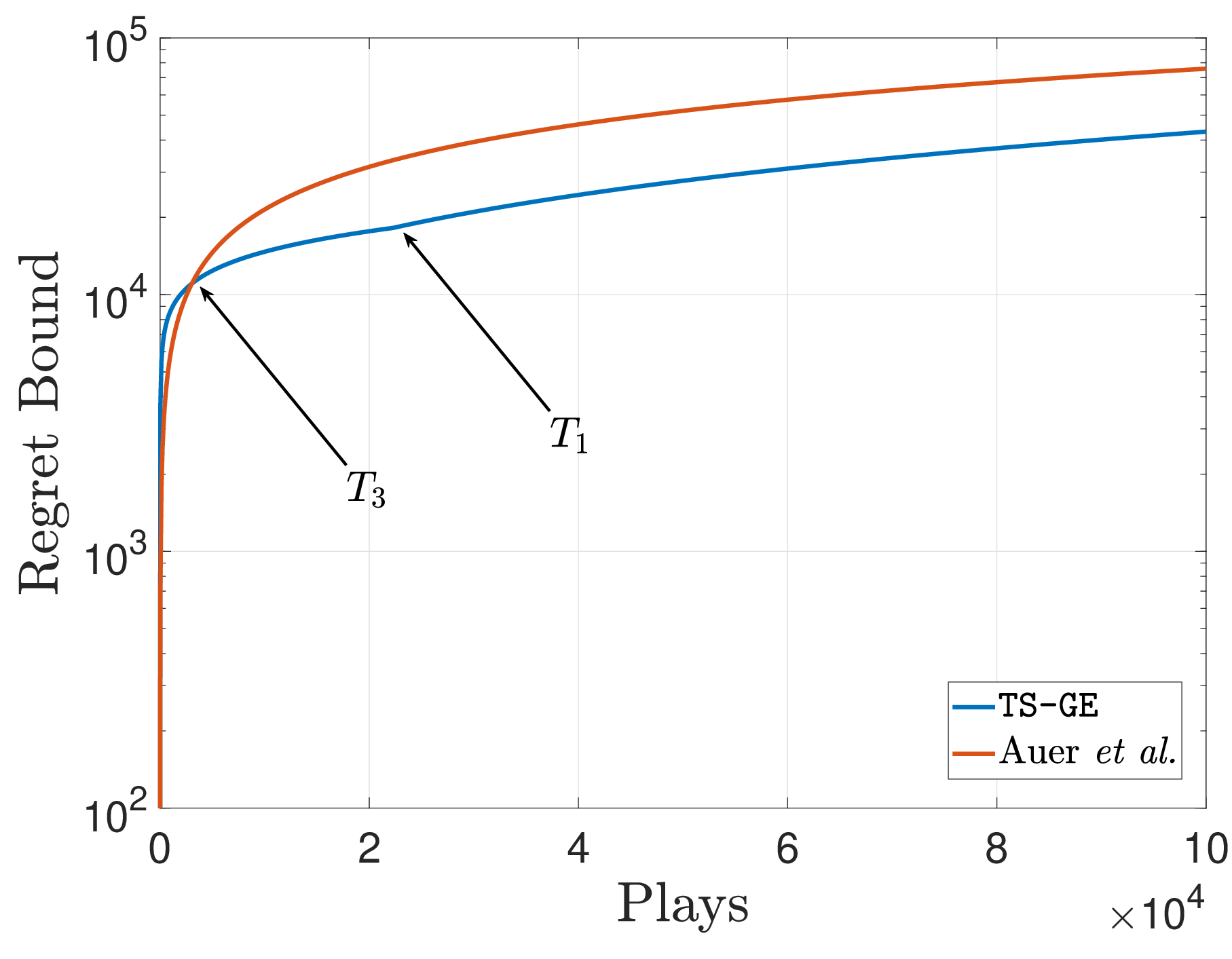}}
  \caption{Comparison of the regret bound of \texttt{TS-GE} with the best known regret bounds of \texttt{ADSWITCH} and \texttt{M-UCB} for different values of $K$}.
\end{figure*}

In this section, we discuss the derived regret bound of \texttt{TS-GE} with respect to the best known bound of \texttt{ADSWITCH}~\cite{auer2019adaptively} and \texttt{M-UCB}~\cite{cao2019nearly}, Let us define three time slots $T_1$, $T_2$. and $T_3$ as follows:
\begin{itemize}
    \item $T_1 \overset{\cdot}{=} t: N_{\rm C}(1 + \log_2 K) = t^{\frac{2}{5}}$. 
    \item $T_2, T_3 \overset{\cdot}{=} t: \mathcal{R}_{\texttt{TS-GE}}(t) = \sqrt{N_{\rm C} K t \log t}$, $T_2 \leq T_3$.
\end{itemize}
In other words, $T_2$ and $T_3$ are the time instants where the regret bound of \texttt{TS-GE} matches the bound of \texttt{ADSWITCH} or \texttt{M-UCB}. We compare the regret bounds for three different number of arms relevant for a massive \ac{IoT} setup: small - $K = 100$ arms, medium - $K = 500$ arms, and large - $K = 1000$ arms. We consider a time-horizon of $T = 1e5$ number of plays.

For $K = 500$, Fig. \ref{fig:comparison_mid} shows that there are specific regions where \texttt{TS-GE} outperforms \texttt{M-UCB}. Beyond $T_2$, \texttt{M-UCB} outperforms \texttt{TS-GE}. The point of interest for our discussion is the exact location of $T_2$ for different values of $K$. For $K = 100$, the value of $T_2$ is low (see Fig.~\ref{fig:comparison_small}) and the regret bound of \texttt{M-UCB} is lower than \texttt{TS-GE} for most part of the time-horizon. However, in case of $K = 1000$, the value of $T_2$ is beyond the time-horizon, and accordingly, beyond time step 5000, throughout the time frame of interest, \texttt{TS-GE} outperforms \texttt{M-UCB}. This highlights the fact that the time-period of interest in a specific application would dictate the choice of a particular algorithm.

\section{Case Studies in Wireless Networks}

\subsection{Case Study 1: Task Offloading in Edge-Computing}
\label{sec:CS1}
Let us consider a \ac{MEC} cluster consisting of multiple edge servers\footnote{The framework can also be implemented in dense cognitive radio networks with appropriate wireless propagation models.}. The set of servers in the cluster is denoted by $\mathcal{K}$. The \ac{MEC} cluster caters to two different classes of users - primary users and secondary users. The primary users are priority users that are guaranteed \ac{QoS} on admission. The association and admission of the primary users with the MEC servers are controlled centrally so as to provide performance guarantees to all the primary users. On the contrary, the secondary users opportunistically offload their tasks to selected servers based on their availability of compute resources. Unlike for the primary users, the secondary user requests can either be admitted or dropped depending on the current task load of the servers. Due to the dynamic traffic and mobility of the users, we consider that the primary user load on the servers changes with time. This dynamically alters the ability of the servers to admit new secondary user tasks. The availability of the servers is related by the {\it workload buffer} which characterizes the pending task load at a server as discussed next. 

\subsubsection{Temporal evolution of the workload buffer}
The central controller updates the primary user and server connectivity periodically in {\it epochs}. At the beginning of each epoch, the central controller updates the status of at most one server. In particular, for the selected server $a_i \in \mathcal{K}$, the controller admits new primary users with fresh requests and disconnects with the primary users that have already been served by the server. The selection of the server under consideration at each epoch follows the current server load, temporally dynamic traffic density, and channel conditions. In a given epoch, the controller may also decide to not update any server. The server selection policy for updates is unknown to the secondary users. Once the update period is over, it is followed by the service period for the primary and the secondary users simultaneously. The service period consists of multiple time-slots which lasts until the beginning of the next epoch. While each primary user is allotted a dedicated server with service guarantee, on the contrary, each secondary user chooses one or multiple servers during the service period to offload their computational task. In case multiple servers are selected by the secondary user, it partitions the task into segments of sub-tasks and sends to each corresponding server. Let the beginning of the $i-$th epoch be denoted by $t_{\rm E}(i) \in [T],$ where $i = 1, 2, 3, \ldots$. For an epoch $l$, we denote the workload buffer size for server $a_k \in \mathcal{K}$ as $B_{0,l}^{k}$. An admission control mechanism is assumed wherein a server $a_k$ accommodates a maximum workload buffer of $W_{k, \max}$. Let $m_l^k$ denote the number of admitted users in the server $a_k$ in epoch $l$. The offload rate for the primary users and the total service rate for the server $a_k$ is assumed to be $\eta$ cycles/second and $C_k$ cycles/second, respectively. In case the $j-$th server changes its state in the $l-$th epoch, the size of the workload buffers evolves during $t_{\rm E}(l) \leq t \leq t_{\rm E}(l)$ as
\begin{align}
    W_{k}(t) = 
    \begin{cases}
    W_{k}(t_{\rm E}(l)) + (m_{l,k} \cdot \eta - C_k)\, (t- t_{E}(l)) \quad k = j \\
    W_{k}(t_{\rm E}(l)) +  - C_k\, (t- t_{E}(l)) \quad k \neq j
    \end{cases}
\end{align}

\subsubsection{Strategy for the secondary users}
We model the task offloading process of a secondary-user in this network as the \ac{MAB} problem. Recall that the secondary users are not aware of the state changes in the servers. However, they have information about the epochs. We map the \texttt{TS-GE} framework developed in this paper to the strategy of a single secondary user\footnote{Consideration of collision among the requests of multiple secondary users at a single server is beyond the scope of this paper and will be treated in a future work.}. In particular, the service period is mapped to the \texttt{TS} phase of the algorithm. Then, the update and association period of the next epoch corresponds to the \ac{BP} and \ac{GE} of an episode.

Let a secondary user attempts to offload a task which takes $\delta$ cycles to process. Based on the workload buffer, the time required by server $a_k$ to generate the output is calculated as:
 \begin{equation}
     \zeta_{l,k} = \frac{W_k(t) + \delta}{C_k}
 \end{equation}
We consider that the offloading is successful if the process output is generated within a compute deadline $T_{\max}$ of the secondary user. Thus, the reward for selecting an arm $a_k$ in the \ac{TS} phase is modeled as $R_{a_k} = T_{\max} - \zeta_{l,k}(t)$ for Algorithm 1. In the \ac{BP} phase, the edge user partitions its compute task into $K$ equal sub-tasks and attempts to offload a sub-task in each edge server. Naturally, the reward obtained by the user in such a case is the sum of the rewards for each sub-task. This naturally blends into the \texttt{TS-GE} algorithm for change detection and server identification.

\subsubsection{Numerical Example}
We employ the \texttt{TS-GE} framework in a \ac{MEC} cluster with 128 servers. The maximum workload buffer is taken as 1 Giga-cycles, while in a given realization for a server $a_k$ the service rate is uniformly sampled from $\sim \mathcal{U}(2, 4)$ GHz. The value of $\eta$ is taken to be 8 Mbps and the task deadline is taken as 0.05 seconds. In order to highlight the change detection framework, we consider fixed change points across different realizations of the algorithms. Specifically, changes occur at episodes $E_i = 30, 60, 90, 120, 150$.
\begin{figure}
    \centering
    \includegraphics[width = 0.4\textwidth]{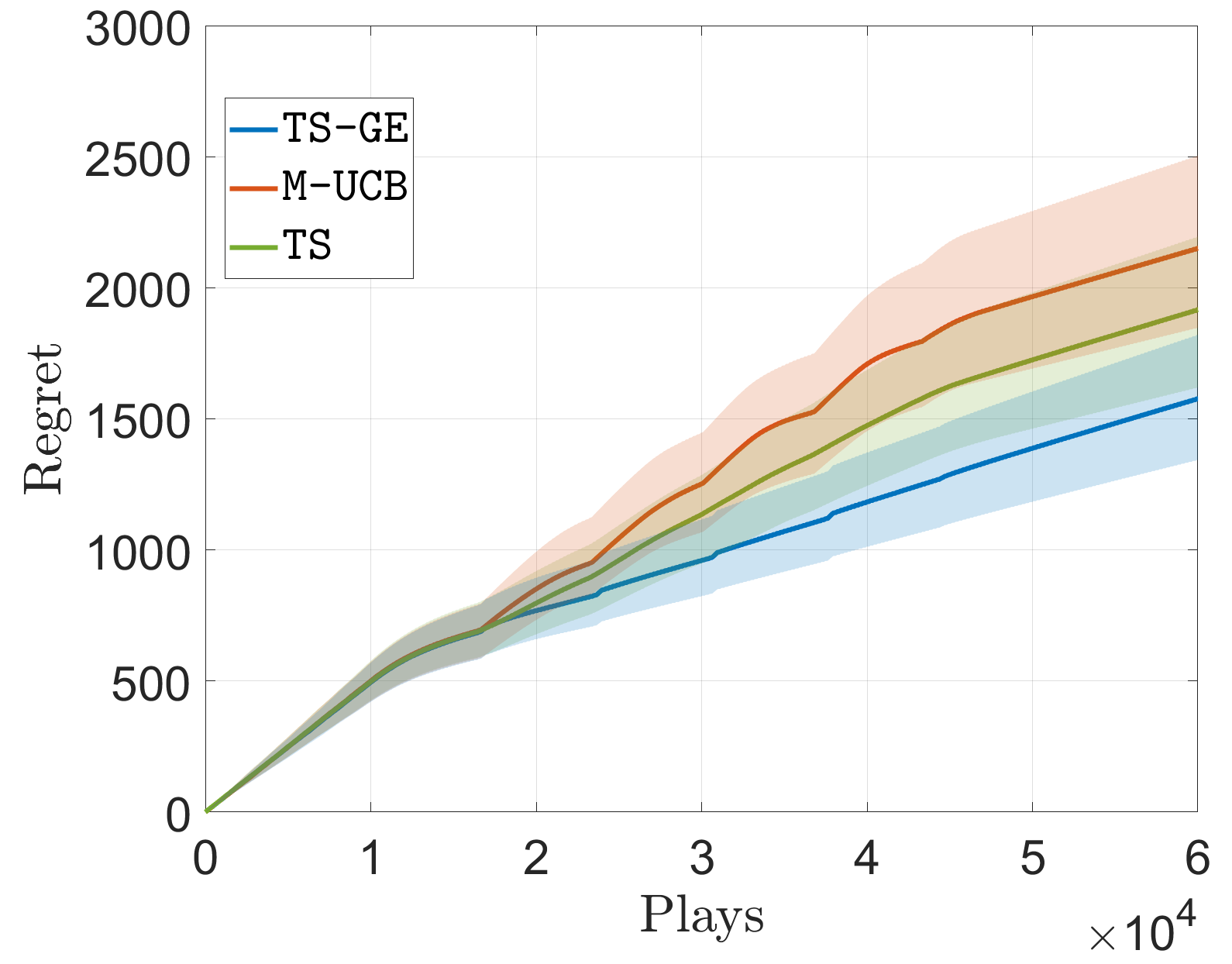}
    \caption{Comparison of the regret of \texttt{TS-GE} with \texttt{M-UCB} and \texttt{TS} for the \ac{MEC} server offloading problem.}
    \label{fig:actualregret}
\end{figure}
In Fig.~\ref{fig:actualregret}, we compare the classical \ac{TS} algorithm with \texttt{TS-GE} and \texttt{M-UCB}.  We see that \texttt{TS-GE} outperforms the others by detecting the exact arm that has undergone a change. Interestingly, \texttt{M-UCB} performs worse than \ac{TS}, mainly due to the fact that \texttt{M-UCB} flushes all the past (potentially relevant) rewards and restarts the exploration procedure once a system-level change is detected.


\subsection{Case Study 2: \ac{SWIPT} in an IIoT Network}
\label{sec:CS2}
In this section, we employ the \texttt{TS-GE} algorithm to evaluate an \ac{IIoT} network where a central controller transmits data-packets to the device with the best channel condition and simultaneously performs wireless power transfer to all the devices.

\subsubsection{Network Model}
Let us consider an \ac{IIoT} network consisting of a central wireless \ac{AP} and $K$ IoT devices. The set of the devices is denoted by $\mathcal{K}$. Typically an industrial environment deals with a large $K$ that represent multiple sensors and cyber-physical systems. The \ac{AP} provisions two wireless services in the network: i) periodic \ac{wpt} to all the IoT devices and ii) a unicast broadband data transmission to one selected IoT device\footnote{The unicast service is relevant in cases when the AP intends to select the best storage-enabled IoT device to transfer large files for caching at the edge. This may then be accessed by the other IoT devices, e.g., using device-to-device link. However, in this paper, we do not delve deeper into such an analysis.}.

Let the scenario of interest be modeled as a two-dimensional disk $\mathcal{B}(0,R)$ of radius $R$ centered around the origin similar to~\cite{ghatak2021stochastic}. The transmit power of the \ac{AP} is $P_{\rm t}$. The location of the devices is assumed to be uniform in $\mathcal{B}(0,R)$. Each AP-device link may be blocked by roaming blockages in the environment. The probability that a link of length $r$ is in \ac{LOS} is assumed to be $p_{\rm L}(r) = \exp(-\omega r)$~\cite{bai2014analysis}. Furthermore, note that due to the presence of a large number of metallic objects, an industrial scenario presents a dense scattering environment. Consequently, we assume that each transmission link experiences a fast-fading $h$ modeled as a Rayleigh distributed random variable with variance 1. Thus, The received power at an IoT device at a distance of $r$ from the \ac{AP} is given by $P_{\rm r}(r) =  K P_{\rm t} h r^{-\gamma}$ with a probability $p_{\rm L}(r)$. Here $K$ and $\gamma$ respectively are the path-loss coefficient and the path-loss exponent. The total transmission bandwidth is assumed to be $B$ which is orthogonally allotted to the users scheduled in one time-slot.

At each episode, the controller selects the device with the best channel conditions and executes information transfer in a sequence of time-slots using the \ac{TS} phase of the algorithm. The device-specific transmission can be facilitated by employing techniques such as beamforming. However we do not consider the details of such procedures.
The \ac{TS} phase information transfer is followed by joint power transfer to all the devices. This is mapped to the \ac{BP} phase of the algorithm. At the end of the \ac{BP} phase, the total energy harvested at all the devices at the end of the \ac{BP} phase is reported back to the controller. Using this total energy transfer report, the controller detects whether a change in the large-scale channel conditions has taken place. If so, then the controller probes multiple devices grouped together as per \texttt{TS-GE} to detect the device with the current best channel conditions.

\subsubsection{Performance Bounds using Stochastic Geometry}
Before proceeding with the evaluation of \texttt{TS-GE} in this network, let us first derive the upper bound on the statistical performance of data-rate. This will enable a comparison with not only an existing algorithm but also the performance limit. Since the location of the devices is assumed to be uniform across the factory floor, they form a realization of a \ac{BPP}. Additionally, due to the assumption that the blockage in each link is independent of each other, the IoT devices are either in LOS or NLOS state. The probability that at least one of the IoT devices is in LOS state is given by:
\begin{align}
B_{\rm L} =  \left[\int_0^R\exp(- \omega t) \frac{2t}{R} dt\right]^K  = \left(2\frac{1 - \exp\left(- \omega R^2\right) \left(\omega R^2 + 1\right)}{\omega^2 R^2}\right)^K \nonumber
\end{align}
The above expression follows similarly to~\cite{bai2014coverage}. On the same lines, the probability that at least one of the IoT devices is in NLOS state is given by
\begin{align}
B_{\rm N} &=  \left(R -  2\frac{1 - \exp\left(- \omega R^2\right) \left(\omega R^2 + 1\right)}{\omega^2 R^2}\right)^K \nonumber 
\end{align}

\subsubsection{Best-link transmission for information transfer}
Out of the possible IoT devices, the central controller selects the device with the best channel condition for information transfer. For that first, let us derive the distance distributions of the nearest LOS and NLOS devices.
\begin{lemma}
The distribution of the distance to the nearest LOS device, $r_{\rm L1}$ and the nearest NLOS device, $r_{\rm N1}$ are respectively:
\begin{align}
    \mathbb{P}\left(r_{\rm L1} \geq x\right) &= \frac{\left(x^2U_{\rm L}(x)\frac{R^2}{R^2 - x^2}\right)^{K+1} - 1}{x^2U_{\rm L}(x)\frac{R^2}{R^2 - x^2} - 1} \left(\frac{R^2 - x^2}{R^2}\right)^K \nonumber \\
    \mathbb{P}\left(r_{\rm N1} \geq x\right) &= \frac{\left(x^2U_{\rm N}(x)\frac{R^2}{R^2 - x^2}\right)^{K+1} - 1}{x^2U_{\rm N}(x)\frac{R^2}{R^2 - x^2} - 1} \left(\frac{R^2 - x^2}{R^2}\right)^K \nonumber \\
\end{align}
where,
\begin{align}
    U_{\rm L}(x) &= \frac{2\left(1 - \exp\left( - \omega x \left( \omega x + 1\right)\right)\right)}{\omega^2 x}, \;
    U_{\rm N}(x) = x - \frac{2\left(1 - \exp\left( - \omega x \left( \omega x + 1\right)\right)\right)}{\omega^2 x} \nonumber 
\end{align}
\end{lemma}
\begin{proof}
Using the void probabilities (see \cite{stoyan2013stochastic}) we have:
\begin{align}
  \mathbb{P}\left(r_{\rm L1} \geq x\right) 
= \sum_{k = 0}^K\left[x^2U_{\rm L}(x)\left(\frac{R^2}{R^2 - x^2}\right)\right]^k \left(\frac{R^2 - x^2}{R^2}\right)^K \nonumber 
\end{align}
Evaluating the above series derives the result. The case for the NLOS device also follows in a similar manner.
\end{proof}


Across different realizations of the network, the best link experiences a fading-averaged downlink received power
\begin{align}
P_{\rm r} = 
\begin{cases}
KP_{\rm t}r_{\rm L1}^{-\gamma_{\rm L}}; &\text{with probability } \mathcal{P}_{\rm L} \\
KP_{\rm t}r_{\rm N1}^{-\gamma_{\rm N}}; &\text{with probability } 1 - \mathcal{P}_{\rm L}
\end{cases}
\label{eq:Pr}
\end{align}
Over a time-horizon of $T$ slots the throughput experienced by the system is
\begin{align}
    \mathcal{T} = \frac{N_{\rm l}T_{\rm TS}}{N_{\rm l} T_{\rm BP} + T_{\rm ETC} + N_{\rm C} T_{\rm GE}}\mathbb{E}\left[B\log_2\left(1 + \frac{P_{\rm r}}{N_0}\right)\right] 
\end{align}
where the expectation taken over $P_{\rm r}$ as per \eqref{eq:Pr}.

\subsubsection{Multicast/Broadcast transmission for power transfer}
Let us assume that in the multicast transmission phase, the \ac{AP} transmits data to a subset $\mathcal{J} \subset \mathcal{K}$ of the IoT devices, where $|\mathcal{J}| = N_J$. In this case, the available bandwidth $B$ is shared among the $N_J$ devices. The harvested power experienced by an IoT device of $\mathcal{J}$ is:
\begin{align}
    \mathcal{T}_j =  
    \begin{cases}
    \theta_e\frac{N_j}{K}{KPr_j^{-\gamma_{\rm L}}}\; \text{ with probability } \exp\left(- \omega r_j\right) \\
    \theta_e\frac{N_j}{K}\frac{BN_0}{N_j}\; \text{ with probability } 1 - \exp\left(- \omega r_j\right)
    \end{cases}
\end{align}
Accordingly, the network sum-energy is given by:
    $\mathcal{T}_T = \sum_{j \in \mathcal{J}} T_j,$
in one slot. In case of the \ac{BP} phase, naturally we have $\mathcal{J} = \mathcal{K}$.

\subsubsection{Numerical Example}
We run the \texttt{TS-GE} algorithm in our IIoT network for a total of 1000 seconds with time-slots of 10 ms~\cite{lin2018synchronization}. Additionally we assume slow moving blockages in which each 30 seconds the visibility state of exactly one IoT device changes from LOS to NLOS or vice-versa.
\begin{figure}
    \centering
    \includegraphics[width = 0.45\textwidth]{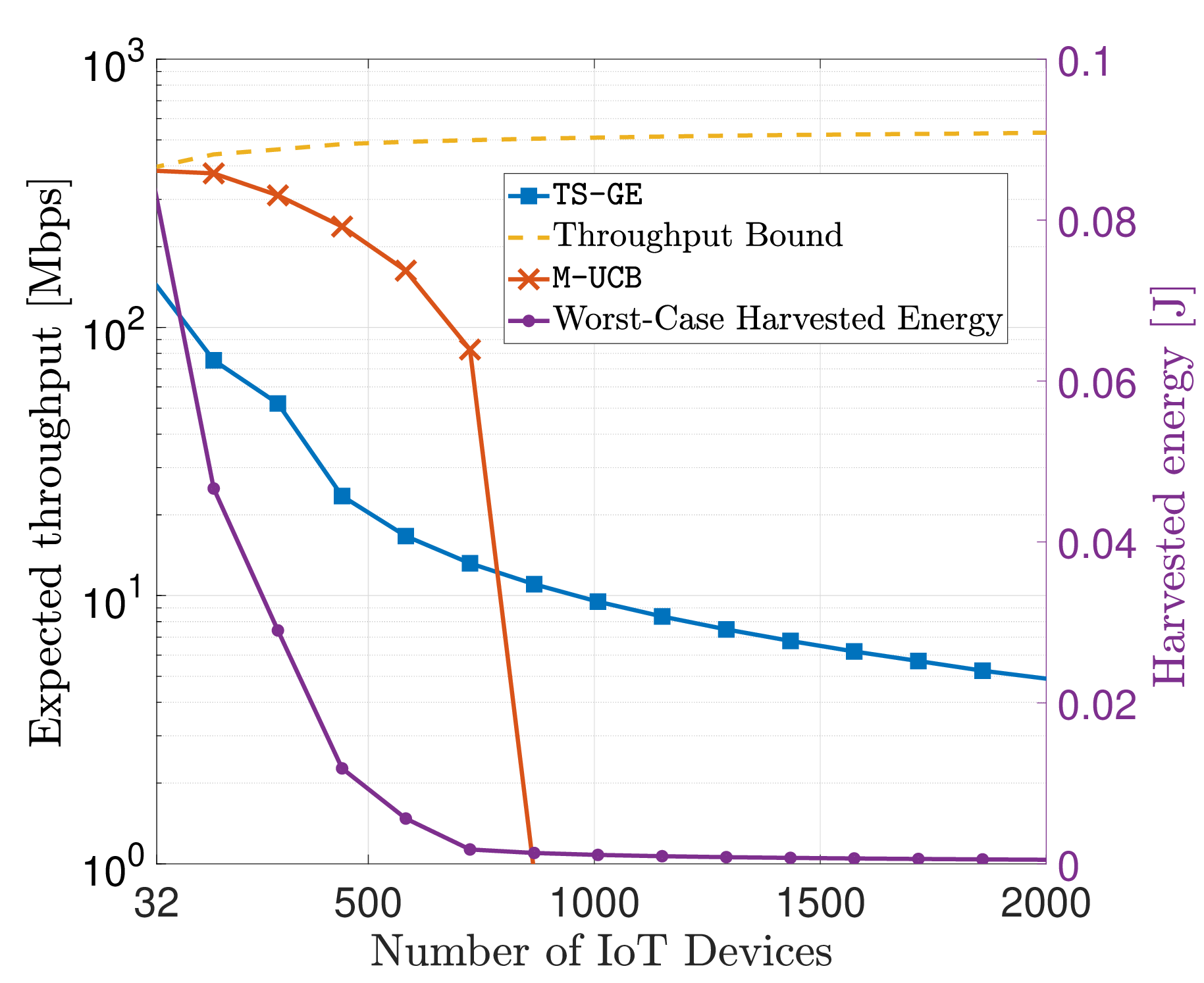}
    \caption{Expected bes-device throughput and worst-case harversted energy with \texttt{TS-GE}.}
    \label{fig:example}
\end{figure}
In Fig.~\ref{fig:example} we plot the average throughput of the information transfer phase (i.e., to the best device) as well as the minimum harvested energy in a device in the \ac{IIoT} network. We observe that in case of a fewer IoT devices, the \texttt{M-UCB} algorithm performs better than \texttt{TS-GE}. Indeed resetting all arms does not incur a large exploration loss in \texttt{M-UCB} in case $K$ is small. Additionally, \texttt{M-UCB} is not constrained by mandatory exploration. Accordingly, it enjoys a higher throughput as compared to \texttt{TS-GE} which needs to transfer energy all the devices.

Interestingly, as the number of devices in the network increases beyond a threshold, \texttt{TS-GE} outperforms \texttt{M-UCB} especially due to rapid changed state identification in the device. Naturally, as $K$ increases, the amount of time dedicated for energy transfer decreases. This is reflected in the reduced energy harvested in the worst device.

Several open problems are apparent. For example, the condition on that each episode can experience only one change can be too stringent to be applied meaningfully in contexts. Furthermore, change detection in the distributions rather than mean and extension to multiple players are indeed interesting directions of research which will be treated in a future work.

\section{Conclusion}
\label{sec:Con}
Existing multi-armed bandit algorithms that are tuned for non-stationary environments sample sub-optimal arms in a probabilistic manner in proportion to their sub-optimality gap. However, several applications require periodic mandatory probing of all the arms. To address this, we develop a novel algorithm called \texttt{TS-GE} which balances the regret guarantees of classical Thompson sampling with a periodic group-exploration phase which not only ensures the mandatory probing of all the arms but also acts as a mechanism to detect changes in the framework. We propose an index-based grouped probing strategy for fast identification of the changed arm. We show that the regret guarantees provided by \texttt{TS-GE} outperforms the state-of-the-art algorithms like \texttt{M-UCB} and \texttt{ADSWITCH} for several time-horizons, especially for a high number of arms. We demonstrated the efficacy of \texttt{TS-GE} in two wireless communication applications - edge offloading and an industrial IoT network designed for simultaneous wireless information and power transfer.

\bibliographystyle{ieeetr}
\bibliography{references}

\appendices

\section{ Proof of Lemma~\ref{lem:PMTS}
}
\label{app:MissTS}
Let the number of times the arm $a_i$ is played is $t_i^-$ times before $t_c$ and $t_i^+$ times after $t_c$. The test statistic (i.e., the parameter to be compared to $2\delta$) is simply a random variable $Z_{TS}$ -
\begin{align}
    &Z_{TS} = \frac{1}{Kn_{\rm ETC}} \sum_{q = (i - 1)n_{\rm ETC} + 1} ^ {in_{\rm ETC}} X_i^-(q) + \frac{1}{K} \sum_{a_j \neq a_i} \frac{1}{n_{\rm ETC}} 
    \sum_{q = (j - 1)n_{\rm ETC} + 1} ^ {jn_{\rm ETC}} X_j(q) + \frac{1}{(K) n_j(T')} \sum_{a_j \neq a_i} \sum_{q_j = T_{\rm ETC} + 1}^{T'}  \cdot \nonumber  \\
     &X_j(q) \mathbb{I}(a_{\texttt{TS-GE}}(q) = a_j)  + \frac{1}{t_i^- + t_i^+} \left[ \sum_{q} X_i^-(q) +\right.  
    \left.  \sum X_i^+(q) \right]  - \frac{1}{KT_{\rm BP}} \sum_{q = T'}^{T''} X_i^+(q) - \frac{1}{KT_{\rm BP}} \sum_{q = T'}^{T''} \sum_{a_j \neq a_i} X_j(q)
\end{align}
Here $T' = T_{\rm ETC} + (m-1) (T_{\rm TS} + T_{\rm BP}) + T_{\rm TS}$, $T'' = T_{\rm ETC} + m(T_{\rm TS} + T_{\rm BP})$, and $n_j(t)$ is the number to times the arm $a_j$ has been played until time $t$. Since all the arms except $a_i$ remain stationary, we have:
\begin{align}
    &\mathbb{P}\left(\Big|\frac{1}{K } \sum_{a_j \neq a_i} \frac{1}{n_{\rm ETC}} \sum_{q = (j - 1)n_{\rm ETC} + 1} ^ {jn_{\rm ETC}} X_j(q) +  \frac{1}{(K) n_j(T')} \right. 
    \left.\sum_{a_j \neq a_i} \sum_{q_j = T_{\rm ETC} + 1}^{T'} X_j(q) \mathbb{I}(a_{\texttt{TS-GE}}(q) = a_j)  +  \right. \nonumber \\
    & \left.  - \frac{1}{KT_{\rm BP}} \sum_{q = T'}^{T''} \sum_{a_j \neq a_i} X_j(q)\Big| \geq 2\delta \right) \leq \mathcal{O}\left(\frac{1}{T^2}\right)
\end{align}

Consequently, for the decision of change detection it is of interest to consider the following random variable instead:
\begin{align}
    Z_{TS}' = &\frac{1}{Kn_{\rm ETC}} \sum_{q = (i - 1)n_{\rm ETC} + 1} ^ {in_{\rm ETC}} X_i^-(q) +  \frac{1}{K(t_i^- + t_i^+)} \left[ \sum_{q} X_i^-(q) + \sum X_i^+(q) \right] - \frac{1}{KT_{\rm BP}} \sum_{q = T'}^{T''} X_i^+(q),
\end{align}
and compare it to a threshold of $2\delta$. Note that $Z'_{TS}$ is Gaussian distributed with mean $\mu_{Z'_{TS}} = \Delta_{{\rm C},i} + \frac{1}{t_i^- + t_i^+}\left(t_i^-\mu_i^- + t_i^+ \mu_I^+\right)$ and variance given by $\sigma_{Z'_{TS}}^2 = \frac{\sigma^2}{K^2}\left(\frac{1}{n_{\rm ETC}} + \frac{1}{t_i^- + t_i^+} + \frac{1}{T_{\rm BP}}\right)$. Consequently, there are two cases of interest:

{\bf Case 1 - $\Delta^C_i >0$:} This is the case where the mean of the arm $a_i$ increases from $\mu_i^-$ to $\mu_i^+$. Accordingly the missed detection probability can be written as:
\begin{align}
    \mathcal{P}_{\rm M}^{\rm TS} &= \mathbb{P}\left(|Z'_{TS}| \leq 2 \delta\right) \leq \mathcal{Q}\left[\frac{\mu_{Z'_{TS}} - 2\delta}{\sigma_{Z'_{TS}}}\right]  \overset{(a)}{\leq}  \mathcal{Q}\left[\frac{\Delta^C_i - 2\delta}{\frac{\sigma}{K}\sqrt{\frac{1}{n_{\rm ETC}} + \frac{1}{t_i^- + t_i^+} + \frac{1}{KT_{\rm BP}}}}\right] \nonumber \\
    & \overset{(b)}\leq  \mathcal{Q}\left[\frac{K\sqrt{n_{\rm ETC} +t_i^- + t_i^+ + T_{\rm BP}}}{3} \frac{\Delta^C_i - 2\delta}{\sigma}\right]  \overset{(c)}\leq \frac{1}{T}. \nonumber 
\end{align}
In the above $\mathcal{Q}(\cdot)$ is the Gaussian-Q function. The inequality (a) follows from the facts that $\frac{1}{t_i^- + t_i^+}\left(t_i^-\mu_i^- + t_i^+ \mu_I^+\right) \geq 0$ and the $\mathcal{Q}(\cdot)$ is a decreasing function. The step (b) follows from the AM $>$ HM inequality, while the step (c) follows from the assumption that $\Delta_{{\rm C},i} \geq 2\sigma$ and the inequality $Q(K . \sqrt{x^{(2/5)}}) \leq \frac{1}{x}$ for $K \geq 1$.

{\bf Case 2 - $\Delta_{{\rm C},i} \leq 0$:} This refers to the case where the mean of arm $a_i$ decreases from $\mu_i^-$ to $\mu_i^+$, i.e., $\mu_i^+ \leq \mu_i^-$. Accordingly, the missed detection probability follows similarly to the above $
    \mathcal{P}_{\rm M}^{\rm TS} = \mathbb{P}\left(|Z'_{TS}| \leq 2 \delta\right)  \leq \frac{1}{T}.$
\section{ Proof of Lemma~\ref{lem:PMBP}
}
\label{app:MissBP}
Let in the BP phase, the number of times all the arms are played simultaneously be $t_i^-$ times before $t_c$ and $t_i^+$ times after $t_c$. Given that other arms $a_j$ where $j \neq i$ have not changed, the test statistic (i.e., the parameter to be compared to $2\delta$) is simply a random variable $Z_{BP}$ similar to $Z_{TS}$. 
given by:
\begin{align}
    &Z_{BP} = \frac{1}{Kn_{\rm ETC}} \sum_{q = (i - 1)n_{\rm ETC} + 1} ^ {in_{\rm ETC}} X_i^-(q) + \frac{1}{K} \sum_{a_j \neq a_i} \frac{1}{n_{\rm ETC}} \sum_{q = (j - 1)n_{\rm ETC} + 1} ^ {jn_{\rm ETC}} X_j(q) +  \frac{1}{K n_j(T')} \sum_{a_j}\nonumber  \\
       &\sum_{q_j = T_{\rm ETC} + 1}^{T'}X_j(q) \mathbb{I}(a_{\texttt{TS-GE}}(q) = a_j)  +  \frac{1}{(t_i^- + t_i^+)K} \left[ \sum_{q} X_i^-(q) + \sum X_i^+(q) \right] - \frac{1}{KT_{\rm BP}} \sum_{q = T'}^{T''} \sum_{a_j \neq a_i} X_j(q)
\end{align}
Since all the arms except $a_i$ remain stationary, we have:
\begin{align}
    &\mathbb{P}\left(\Big|\frac{1}{K } \sum_{a_j \neq a_i} \frac{1}{n_{\rm ETC}} \sum_{q = (j - 1)n_{\rm ETC} + 1} ^ {jn_{\rm ETC}} X_j(q) +  \frac{1}{K n_j(T')}\right. \left.\sum_{a_j \neq a_i} \sum_{q_j = T_{\rm ETC} + 1}^{T'} X_j(q) \mathbb{I}(a_{\texttt{TS-GE}}(q) = a_j)  +  \right. \nonumber \\
    & \left.  - \frac{1}{KT_{\rm BP}} \sum_{q = T'}^{T''} \sum_{a_j \neq a_i} X_j(q)\Big| \geq 2\delta \right) \leq \mathcal{O}\left(\frac{1}{T^2}\right)
\end{align}
Consequently, for the decision of change detection it is of interest to consider the following random variable instead:
\begin{align}
    Z_{BP}' = &\frac{1}{Kn_{\rm ETC}} \sum_{q = (i - 1)n_{\rm ETC} + 1} ^ {in_{\rm ETC}} X_i^-(q) + \frac{1}{Kn_i(T')}\sum_{q} X_i^-(q)   \nonumber \\
    & - \frac{1}{K(t_i^- + t_i^+)} \left[ \sum_{q} X_i^-(q) + \sum X_i^+(q) \right]
\end{align}

Note that $Z'_{BP}$ is Gaussian distributed with mean $\mu_{Z'_{BP}} = \mu_i^- - \frac{1}{t_i^- + t_i^+}\left(t_i^-\mu_i^- + t_i^+ \mu_i^+\right)$ and variance given by $\sigma_{Z'_{TS}}^2 = \frac{\sigma^2}{K^2}\left(\frac{1}{n_{\rm ETC}} + \frac{1}{t_i^- + t_i^+} + \frac{1}{T_{\rm BP}}\right)$.

{\bf Case 1 - $\Delta_{{\rm C},i} > 4\delta$ and $t_i^- \leq \frac{T_{\rm BP}(\Delta_{{\rm C},i} - 4 \delta)}{\Delta_{{\rm C},i}}$:}

This case occurs with a \textit{high} probability. Due to the fact that for this case, we have $t_i^- \leq \frac{T_{\rm BP}(\Delta_{{\rm C},i} - 4 \delta)}{\Delta_{{\rm C},i}}$, i.e., $\mu_{Z'_{BP}} \geq 4 \delta$, thus, similar to the Lemma 2,
    $\mathcal{P}_{M | \text{Case 1}}^{\rm BP} \leq \frac{1}{T}$
Thus, we have
  $\mathcal{P}_{M,\text{Case 1}}^{\rm BP} = \mathcal{P}_{M | \text{Case 1}}^{\rm BP} \cdot \mathcal{P}_\text{Case 1} \leq \frac{1}{T}.$

{\bf Case 2 - $\Delta_{{\rm C},i} > 4\delta$ and $t_i^- > \frac{T_{\rm BP}(\Delta_{{\rm C},i} - 4 \delta)}{\Delta_{{\rm C},i}}$:} Here we have $\mu_{Z'_{BP}} < 4 \delta$, and accordingly, the probability of missed detection is high. However, let us first observe the probability that the change occurs such that $t_i^- > \frac{T_{\rm BP}(\Delta_{{\rm C},i} - 4 \delta)}{\Delta_{{\rm C},i}}$. We have:
\begin{align}
    \mathbb{P}\left(\text{Case 2}\right) &= \mathbb{P}\left(t_i^- > \frac{T_{\rm BP}(\Delta_{{\rm C},i} - 4 \delta)}{\Delta_{{\rm C},i}}\right) = \mathbb{P}\left(\frac{t_i^+}{t_i^-} \geq \frac{4\delta}{\Delta_{{\rm C},i} - 4 \delta}\right) \overset{(a)}\leq \frac{1}{T} \nonumber 
\end{align}
where the step (a) is due to Assumption \ref{ass:pC}.

{\bf Case 3 - $\Delta_{{\rm C},i} < 4\delta$ and $t_i^- \leq \frac{T_{\rm BP}(-\Delta_{{\rm C},i} - 4 \delta)}{\Delta_{{\rm C},i}}$:} This case is similar to Case 1 and hence we skip the detailed proof for brevity. In summary, similar to Case 1, for Case 3 the probability that the change occurs at a time step such that $t_i^- \leq \frac{T_{\rm BP}(-\Delta_{{\rm C},i} - 4 \delta)}{\Delta_{{\rm C},i}}$ holds is high. However, the probability of missed detection is bounded by $\frac{1}{T}$.

{\bf Case 4 - $\Delta_{{\rm C},i} < 4\delta$ and $t_i^- > \frac{T_{\rm BP}(-\Delta_{{\rm C},i} - 4 \delta)}{\Delta_{{\rm C},i}}$:} This is similar to Case 2, wherein the probability of missed detection is high, while due to the Assumption \ref{ass:pC}, the occurrence of the change such that the condition $t_i^- > \frac{T_{\rm BP}(-\Delta_{{\rm C},i} - 4 \delta)}{\Delta_{{\rm C},i}}$ holds is bounded by $\frac{1}{T}$.

\section{Proof of Theorem~\ref{theo:Regret}}
\label{app:Regret}
Recall that Each episode either experiences one change or no changes.

{\bf Regret in case of no change:}
The number of such episodes is $N_{\rm l} - N_{\rm C}$. Each such episode experiences a mandatory regret bounded by:
\begin{align}
    \mathcal{R}_{\text{no change}}^1(T_{\rm l}) \leq \underbrace{\mathcal{O}\left[\log \left(\sqrt{T} - T^{\frac{2}{5}}\right)\right]}_{A} + \underbrace{\mathcal{O}\left(T^{\frac{2}{5}}\right)}_{B}, \nonumber
\end{align}
where the term $A$ is due to the TS phase (e.g., see~\cite{agrawal2012analysis}) and the term B is due to the BP phase, where we assume linear regret as the worst case. In case of a false alarm, the algorithm subsequently experiences worst-case regret in all the subsequent phases. This occurs with a probability of $\mathcal{P}_{FA}$, and hence its contribution to the overall regret is
  $\mathcal{R}_{\text{no change}}^2(T_{\rm l}) \leq \mathcal{P}_{FA} \Delta_{max} T \overset{(a)}{\leq} K_1$,
where $\Delta_{\rm max}$ is the maximum difference between the means of the arms at any given instant.
The step (a) follows from \eqref{eq:pfa}. Thus, overall, for the case of no change, the regret is
    $\mathcal{R}_{\text{no change}}(T_{\rm l}) \leq  \mathcal{O}\left[\log \left(\sqrt{T} - T^{\frac{2}{5}}\right)\right] + \mathcal{O}\left(T^{\frac{2}{5}}\right) + K_1.$

{\bf Regret in case of change:}
The number of such episodes is $N_{\rm C}$. Each such episode experiences a regret bounded by
$\mathcal{R}_{\text{change}}^1 \leq \mathcal{O}\left(\Delta_{max} \sqrt{T}\right).$
where we assume that in case a change occurs, the player suffers a worst-case linear regret for the rest of the episode. However, in case of missed detection, the algorithm experiences worst-case regret in all the subsequent phases. This occurs with a probability $\mathcal{P}_{\rm M} = p_{\rm C}^{\rm TS} \mathcal{P}_{M}^{\rm TS} + p_{\rm C}^{\rm BP}\mathcal{P}_{\rm M}^{\rm BP}$ which varies as $\frac{1}{T}$ and hence its contribution to the overall regret is $\mathcal{R}_{\text{change}}^2 \leq \mathcal{P}_{\rm M} \Delta_{max} T = K_2$.
Thus, overall, for the case of no change, the regret    $\mathcal{R}_{change} \leq \mathcal{O}\left(\Delta_{max} \sqrt{T}\right) + K_2.$
Using the above development, we can bound the regret of \texttt{TS-GE} as follows:
\begin{align}
    \mathcal{R}(T) &=  \mathcal{R}_{ETC}\left(T_{\rm ETC}\right) + \sum_{i = 1}^{N_{\rm l}} \mathcal{R}_i(T_{\rm l}) \nonumber \\ 
   &\leq \mathcal{O}\left(K\log T\right) + (N_{\rm l} - N_{\rm C})\mathcal{R}_{\text{no change}} + N_{\rm C} \mathcal{R}_{\text{change}} \nonumber \\
   &\leq\mathcal{O}\left( K\log T + \sqrt{T}\left[\max\{N_{\rm C}\left(1 + \log K\right), T^{\frac{2}{5}}\}\right]\right)
\end{align}

\end{document}